%% file: main.tex
\newcommand{\shortversion}[1]{}
\newcommand{\longversion}[1]{#1}

\documentclass[letterpaper, 10 pt, conference]{ieeeconf}  

\IEEEoverridecommandlockouts                              

\usepackage{amsmath}        
\usepackage{amsfonts}
\usepackage{xcolor}
\usepackage[linesnumbered,ruled,noend,noline]{algorithm2e}
\usepackage{romannum}
\usepackage{graphicx}
\usepackage{caption}
\usepackage{subcaption}
\usepackage[export]{adjustbox}
\usepackage[switch]{lineno}
\usepackage{multirow}
\usepackage{booktabs}
\usepackage{algorithmicx}
\usepackage{fbox}           
\usepackage{dashbox}        
\usepackage{siunitx}        
\usepackage{soul}           
\usepackage{xcolor}         
\usepackage{tikz}           
\usepackage{multicol}       
\usepackage{rotating}       

\usetikzlibrary{fit,calc}   

\graphicspath{{figures/}}

\SetCommentSty{mycommfont}

\SetKw{Break}{break}
\SetKw{Continue}{continue}
\SetKw{Exit}{exit}

\SetKwFunction{FMain}{Main}
\SetKwFunction{FRcvLV}{receive\_localview}
\SetKwFunction{FSndPath}{send\_path}
\SetKwFunction{FRcvAck}{receive\_ack}
\SetKwFunction{FInvoke}{check\_invocation}
\SetKwFunction{FnCheckCPPCriteria}{check\_CPP\_criteria}
\SetKwFunction{FnOnDemCPP}{OnDemCPP}
\SetKwFunction{FnOnDem}{OnDem}
\SetKwFunction{FnConCPP}{ConCPP}
\SetKwFunction{FnCon}{Con}
\SetKwFunction{FnAPFCPP}{APFCPP}
\SetKwFunction{FnAPF}{APF}
\SetKwFunction{FnConCPPRound}{ConCPP\_Round}
\SetKwFunction{FnCPPForPar}{CPPForPar}
\SetKwFunction{FnConCPPForPar}{ConCPPForPar}
\SetKwFunction{FnCOPForPar}{COPForPar}
\SetKwFunction{FnCFPForPar}{CFPForPar}
\SetKwFunction{FnGAMRCPP}{GAMRCPP}
\SetKwFunction{FnGAMRCPPHorizon}{GAMRCPP\_Horizon}

\SetKwFunction{FSndPathToRob}{send\_path\_to\_rob}
\SetKwFunction{FnSendPathsToActiveParticipants}{send\_paths}
\SetKwFunction{FnIncrementEta}{increment\_eta}

\SetKwProg{Fn}{Function}{:}{}
\SetKwProg{Srv}{Service}{:}{}

\SetKwFunction{SrvRcvLV}{receive\_localview}

\definecolor{mygray}{gray}{0.9}
\sethlcolor{mygray}

\newtheorem{example}{Example}
\newtheorem{problem}{Problem}
\newtheorem{theorem}{Theorem}
\newtheorem{lemma}{Lemma}

\title{\LARGE \bf
Online Concurrent Multi-Robot Coverage Path Planning
}

\author{Ratijit Mitra$^{1}$ and Indranil Saha$^{2}$
\thanks{$^{1}$Ratijit Mitra and $^{2}$Indranil Saha are with the Department of Computer Science and Engineering, Indian Institute of Technology Kanpur, Uttar Pradesh - 208016, India
        {\tt\small \{ratijit, isaha\}@iitk.ac.in}}%
}

\begin{document}

\maketitle
\thispagestyle{empty}
\pagestyle{empty}

\begin{abstract}
\input{1_abstract}
\end{abstract}

\input{2_introduction}
\input{3_problem}
\input{4_solution}
\input{5_theoretical_analysis}
\input{6_evaluation}
\input{7_conclusion}


\bibliographystyle{IEEEtran}
\bibliography{8_references}

\end{document}

%% file: 1_abstract.tex
Recently, centralized receding horizon online multi-robot coverage path planning algorithms have shown remarkable scalability in thoroughly exploring large, complex, unknown workspaces with many robots. 
In a horizon, the path planning and the path execution interleave, meaning when the path planning occurs for robots with no paths, the robots with outstanding paths do not execute, and subsequently, when the robots with new or outstanding paths execute to reach respective goals, path planning does not occur for those robots yet to get new paths, leading to wastage of both the robotic and the computation resources. 
As a remedy, we propose a centralized algorithm that is not horizon-based. 
It plans paths at any time for a subset of robots with no paths, i.e., who have reached their previously assigned goals, while the rest execute their outstanding paths, thereby enabling concurrent planning and execution. 
We formally prove that the proposed algorithm ensures complete coverage of an unknown workspace and analyze its time complexity. 
To demonstrate scalability, we evaluate our algorithm to cover eight large $2$D grid benchmark workspaces with up to $512$ aerial and ground robots, respectively. 
A comparison with two state-of-the-art horizon-based algorithms shows its superiority in completing the coverage with up to $1.6\times$ speedup. 
For validation, we perform ROS + Gazebo simulations in six $2$D grid benchmark workspaces with $10$ Quadcopters and TurtleBots, respectively. 
We also successfully conducted one outdoor experiment with three quadcopters and one indoor with two TurtleBots. 

%% file: 2_introduction.tex
\section{Introduction}
\label{sec:introduction}

Coverage Path Planning (CPP) determines conflict-free paths for a group of mobile robots so they can visit obstacle-free sections of a particular workspace to perform a given task. 
It has many indoor (e.g., cleaning \cite{DBLP:conf/icra/BormannJHH18} and industrial inspection \cite{DBLP:conf/iros/JingPGRLS17}) and outdoor (e.g., precision agriculture \cite{DBLP:journals/jfr/BarrientosCCMRSV11}, \cite{DBLP:journals/ral/MainiGI22}, surveying \cite{DBLP:conf/icra/KarapetyanMLLOR18}, and disaster management \cite{drones3010004}) applications. 
A CPP algorithm is also known as a Coverage Planner (CP), and its primary objective is to ensure \textit{complete} coverage, meaning the robots visit the entire obstacle-free section of the workspace. 
While a single robot (e.g., \cite{DBLP:conf/icra/GabrielyR01, DBLP:conf/iros/KleinerBKPM17, DBLP:conf/atal/SharmaDK19}) can cover a small workspace, it may take too long to cover a large workspace. 
Multiple robots (e.g., \cite{DBLP:conf/icra/ModaresGMD17, DBLP:conf/iros/KarapetyanBMTR17, DBLP:conf/icra/VandermeulenGK19, DBLP:conf/iros/HardouinMMMM20, DBLP:conf/icra/CollinsGEDDC21, DBLP:conf/icra/TangSZ21}) can cover a large workspace faster if the CP distributes the workload fairly among the robots. 
However, designing a CP that can supervise many robots efficiently is challenging. 

We call a CP \textit{offline} if the workspace is \textit{known}, i.e., the CP knows the obstacle placements beforehand (e.g., \cite{DBLP:journals/ras/GalceranC13, DBLP:conf/iros/BouzidBS17, DBLP:conf/iros/ChenTKV19}, \cite{DBLP:journals/ral/LuZTLW23}). 
So, it finds the robots' paths all at once. 
In contrast, we call a CP \textit{online} if the workspace is \textit{unknown}, i.e., the CP has no prior knowledge about the obstacle placements (e.g., \cite{DBLP:conf/agents/Yamauchi98, DBLP:conf/icra/HazonMK06, DBLP:conf/icra/OzdemirGKHG19, DBLP:conf/icra/DharmadhikariDS20}, \cite{DBLP:journals/ral/KanTK20}). 
So, it runs numerous iterations to cover the workspace gradually. 
The robots send sensed data, acquired through attached sensors, about the explored workspace sections to the CP in each iteration, and the CP finds subpaths that the robots follow to visit obstacle-free sections not visited so far and explore more of the workspace. 

We classify a CP as \textit{centralized} (e.g., \cite{DBLP:conf/agents/Yamauchi98, DBLP:conf/icra/GabrielyR01, DBLP:conf/iros/KleinerBKPM17, DBLP:conf/iccps/DasS18, DBLP:conf/atal/SharmaDK19}), which runs in a server and solely obtains all the robots' paths, or \textit{distributed} (e.g., \cite{DBLP:conf/icra/HazonMK06, DBLP:journals/apin/VietDCC15}, \cite{DBLP:journals/ral/ZhongCHO19}), which runs in each robot as a local instance. 
These local instances \textit{collaborate} among themselves to obtain individual paths. 
Recently, for large unknown workspaces, centralized \textit{horizon}-based (e.g., \cite{DBLP:conf/icra/BircherKAOS16, DBLP:conf/iccps/DasS18}) algorithm \FnGAMRCPP\cite{DBLP:conf/iros/MitraS22} has shown its \textit{scalability} in finding shorter length paths for many robots (up to 128) by exploiting the global state space, thereby completing the coverage faster than distributed algorithm $\mathtt{BoB}$~\cite{DBLP:journals/apin/VietDCC15}. 

Centralized online multi-robot horizon-based CPs, like \cite{DBLP:conf/iccps/DasS18, DBLP:conf/iros/MitraS22, DBLP:conf/icra/MitraS24}, \cite{DBLP:journals/ral/WangZZYWTZS24}, in each horizon, attempts to plan paths for either all the robots (e.g., \cite{DBLP:conf/iccps/DasS18, DBLP:conf/iros/MitraS22}, \cite{DBLP:journals/ral/WangZZYWTZS24}) or a subset of robots (e.g., \cite{DBLP:conf/icra/MitraS24}), called \textit{the participants}, who have entirely traversed their last planned paths. 
So, the participants do not have any remaining paths to follow, demanding new paths from the CP. 
The rest (if any), called \textit{the non-participants}, are yet to finish following their last planned paths, so they have remaining paths left to follow. 
Due to the \textit{interleaving} of \textit{path planning} (by the CP) and subsequently \textit{path following} (by the robots) in each horizon of these CPs, they suffer from the \textit{wastage of robotic resources} 
in two ways. 
When the CP replans for the participants, the non-participants wait until the planning ends. 
Another, while the \textit{active participants} (the participants for whom the CP has found the paths) and the non-participants of the current horizon follow their paths, the \textit{inactive participants} wait for reconsideration by the CP in the next horizon till the path following is over. 
The latter is also a \textit{waste of computation resources}. 

In this paper, we design a centralized online multi-robot non-horizon-based CP, where planning and execution of paths happen in \textit{parallel}. 
It makes the CP design extremely challenging because while planning the participants' conflict-free paths on demand, it must accurately infer the remaining paths of the non-participants, which are moving in parallel. 
The proposed CP \textit{timestamps} the active participants' paths with a \textit{discrete global clock} value, meaning the active participants must follow respective paths from that time point. 
Hence, path execution happens in \textit{sync} with the global clock. 
However, path planning happens anytime if there is a participant. 
Thus, the CP enables overlapping path planning with path execution, significantly reducing the time required to complete the coverage with hundreds of robots. 

We formally prove that our CP ensures complete coverage of an unknown workspace and then analyze its time complexity. 
We evaluate our CP on eight large $2$D grid benchmark workspaces with up to $512$ aerial and ground robots, respectively. 
The results justify its scalability by outperforming the state-of-the-art \FnOnDemCPP\cite{DBLP:conf/icra/MitraS24} and \FnAPFCPP\cite{DBLP:journals/ral/WangZZYWTZS24}, completing the coverage of large unknown workspaces faster with hundreds of robots. 
We validate the CP through ROS + Gazebo simulations on six $2$D grid benchmark workspaces with $10$ aerial and ground robots, respectively. 
Moreover, we perform two experiments with real robots --- 
One outdoor experiment with three quadcopters and another indoor with two TurtleBots. 

%% file: 3_problem.tex
\section{Problem}
\label{sec:problem}

\subsection{Preliminaries}
\label{subsec:preliminaries}

Let $\mathbb{R}$ and $\mathbb{N}$ denote the set of real numbers and natural numbers, respectively, and $\mathbb{N}_0$ denotes the set $\mathbb{N} \cup \{0\}$. 
Also, for $m \in \mathbb{N}$, $[m] = \{n \in \mathbb{N}\, |\, n \leq m\}$ and $[m]_0 = [m] \cup \{0\}$. 
The size of the countable set $\mathcal{S}$ is denoted by $|\mathcal{S}| \in \mathbb{N}_0$. 
Furthermore, we denote the Boolean values $\{0, 1\}$ by $\mathbb{B}$.

\subsubsection{Workspace}
\label{subsubsec:workspace}

We consider an unknown $2$D grid workspace $W$ of size $X \times Y$, where $X, Y \in \mathbb{N}$ are its size along the $x$ and the $y$ axes, respectively. 
So, $W$ consists of \textit{square-shaped equal-sized non-overlapping} cells $\{(x, y)\ |\ x \in [X] \land y \in [Y]\}$. 
Some of these cells are obstacle-free (denoted by $W_{free}$), hence traversable. 
The rest are \textit{fully} occupied with \textit{static} obstacles (denoted by $W_{obs}$), hence not traversable. 
Note that $W_{free}$ and $W_{obs}$ are unknown initially, and $W = W_{free} \cup W_{obs}$ and $W_{free} \cap W_{obs} = \emptyset$. 
We assume that $W_{free}$ is 
\textit{connected}.

\subsubsection{Robots and their States}
\label{subsubsec:robots_and_their_motions}

We employ a team of \mbox{$R \in \mathbb{N}$}  \textit{homogeneous failure-free} robots, where each robot fits entirely within a cell. 
We denote the $i (\in [R])$-th robot by $r^i$ and assume that the robots are \textit{location-aware}. 
So, we denote the state of $r^i$ at the $j (\in \mathbb{N}_0)$-th \textit{discrete} time step by $s^i_j$, which is a tuple of its location and possibly orientation in $W$. 
We define a function $\mathcal{L}$, which takes a state $s$ as input and returns the corresponding location $\mathcal{L}(s) \in W$ as a tuple. 
Initially, the robots get randomly deployed at different cells in $W_{free}$, comprising \textit{the set of start states} $S = \{s^i_0\ |\ i \in [R] \land \mathcal{L}(s^i_0) \in W_{free} \land \forall j(\neq i) \in [R]\ \mathcal{L}(s^i_0) \neq \mathcal{L}(s^j_0)\}$. 
We also assume that each $r^i$ is fitted with four rangefinders on all four sides to detect obstacles in those neighboring cells.

\subsubsection{Motions and generated Paths}
\label{subsubsec:motions_and_generated_paths}

The robots have \textit{a set of motion primitives} $M$ to change their states in each time step. 
It contains a unique motion primitive $\mathtt{Halt(H)}$ to keep a state unchanged in the next step. 
Each motion primitive $\mu \in M$ is associated with some cost $\mathtt{cost(\mu)} \in \mathbb{R}$ (e.g., distance traversed, energy consumed, etc.), but we assume that they all take the same $\tau \in \mathbb{R}$ amount of time for execution. 

Initially, the path $\pi^i$ of robot $r^i$ contains its start state $s^i_0$. 
So, the \textit{length} of $\pi^i$, denoted by $|\pi^i|$, is $0$. 
When a sequence of motion primitives $(\mu_j \in M)_{j \in [\Lambda]}$ of length $\Lambda \in \mathbb{N}$ gets applied on $s^i_0$, it results in a $\Lambda$-length path $\pi^i$, i.e., $|\pi^i| = \Lambda$. 
Therefore, $\pi^i$ contains a sequence of states of $r^i$ s.t. \mbox{$s^i_{j - 1} \xrightarrow{\mu_j} s^i_j, \forall j \in [\Lambda]$}. 
Note that we can make all the path lengths equal to $\Lambda$ by suitably applying $\mathtt{H}$ motions at their ends. 
The resultant \textit{set of paths} is \mbox{$\Pi = \{\pi^i\ |\ i \in [R] \land |\pi^i| = \Lambda\}$}, where each $\pi^i$ satisfies the following three conditions: 
    \begin{enumerate}
        \item $\forall j \in [\Lambda]_0\ \mathcal{L}(s^i_j) \in W_{free}$,\ \ \ \ \ \ \ \ \ \ \ \ \ \ \ \ \ [\textit{Avoid obstacles}]
        \item $\forall j \in [\Lambda]_0\ \forall k \in [R] \setminus \{i\}\\
        \mathcal{L}(s^i_j) \neq \mathcal{L}(s^k_j)$,\ \ \ \ \ \ \ \ \ \ \ \ \ \ \ \ \ [\textit{Avert same cell collisions}]
        \item $\forall j \in [\Lambda]\ \forall k \in [R] \setminus \{i\}\\
        \lnot((\mathcal{L}(s^i_{j-1}) = \mathcal{L}(s^k_j)) \wedge (\mathcal{L}(s^i_j) = \mathcal{L}(s^k_{j-1})))$.\\
        \textcolor{white}{BLANK}\ \ \ \ \ \ \ \ \ \ \ \ \ \ \ \ \ \ \ \ \ \ \ \ \ \ \ \ \ [\textit{Avert head-on collisions}]
    \end{enumerate}

Further, we define the cost of a path $\pi^i$, denoted by $\mathtt{cost}(\pi^i) \in \mathbb{R}$, as the sum of its motions' costs, i.e., $\mathtt{cost}(\pi^i) = \sum_{j \in [\Lambda]} \mathtt{cost}(\mu_j)$. 
Note that, for simplicity, we do not consider \textit{diagonal} motions that cause \textit{side} collisions.

\begin{example}
To illustrate using an example, we consider a TurtleBot~\cite{key_tb}, which rotates around its axis and drives forward. 
So, the state of a TurtleBot is $s^i_j = (x, y, \theta)$, where $(x, y) \in W$ is its location and $\theta \in \{\mathtt{East(E)}, \mathtt{West(W)}, \mathtt{North(N)}, \mathtt{South(S)}\}$ is its orientation in $W$. 
The motions of a TurtleBot are $M = \{\mathtt{Halt(H)}, \mathtt{TurnRight(TR)}, \mathtt{TurnLeft(TL)}, \mathtt{MoveNext(MN)}\}$, where $\mathtt{TR}$ and $\mathtt{TL}$ rotate a TurtleBot $90^{\circ}$ \textit{clockwise} and \textit{counterclockwise}, respectively, and $\mathtt{MN}$ moves it to the next cell pointed by its orientation $\theta$. 
\end{example}

\subsection{Problem Definition}
\label{subsec:problem_definition}

\begin{problem}[Complete Coverage Path Planning Problem]
Formally, it is defined as a problem $\mathcal{P}_{CPP}$, which takes an unknown workspace $W$, a team of $R$ robots with their start states $S$ and motions $M$, and generates their \textit{obstacle avoiding} and \textit{inter-robot collision averting} paths $\Pi$ to ensure each obstacle-free cell gets visited by at least one robot:
\begin{center}
$\Pi = \mathcal{P}_{CPP}(W, R, S, M)$ s.t. $\cup_{i \in [R]} \cup_{j \in [\Lambda]_0} \mathcal{L}(s^i_j) = W_{free}$. 
\end{center}
\end{problem}

%% file: 4_solution.tex
\section{Concurrent CPP Framework}
\label{sec:concpp}

This section presents the proposed centralized concurrent CPP framework that makes a team of mobile robots cover an unknown workspace completely, whose size and boundary are only known to the CP and the robots. 

To ensure space-time consistency of the robots' paths, the CP runs a \textit{global discrete} clock $CLK \in \mathbb{N}_0$ that increments after each $\tau\ \si{\second}$ (the motion primitive execution time) and the clocks on the robots are in sync with it using a clock synchronization protocol~\cite{Eidson20}. 
Each robot has its initial perception of the workspace, called \textit{local view}, 
as the fitted rangefinders are range-limited. 
Initially, all the robots \textit{participate} in planning, so they send their local views to the CP through \textit{request} messages for paths. 
When the CP receives any request, it immediately updates the \textit{global view} of the workspace using the local view. 
After receiving the intended number of requests, the CP uses the global view to generate paths, possibly of different lengths, for the participants. 
The participants for which the CP has found paths are said to be \textit{active} while the rest are said to be \textit{inactive}. 
The CP, while generating the paths, 
\textit{timestamps} them with a particular $CLK$ value, signifying the active participants must follow their respective path from that specified $CLK$ value. 
Now, the CP \textit{simultaneously} sends paths to those active participants through \textit{response} messages. 
The robots that receive responses follow their respective paths from the timestamped $CLK$ and update their local views accordingly. 
As a robot finishes following its path, it sends its updated local view to the CP. 
Upon receiving, first, the CP updates the global view, and then it begins the next planning round if some \textit{preset} conditions are satisfied. 
At the end of any current planning round, the CP checks the preset conditions to begin the next planning round for inactive participants, new participants, or both. 
Thus, the path planning is \emph{on-demand}, where the CP replans for the participants who have entirely traversed their previously planned paths but without modifying the non-participants' existing paths. 
And, the path following is synchronous w.r.t. $CLK$. 
Moreover, path planning for the participants can happen in \textit{parallel} with the path execution by the non-participants, hence concurrent CPP. 
Note that the CP allows only one instance of planning round at a time, which can begin or end at any time. 

\begin{figure}[t]
    \centering
    \includegraphics[scale=0.34]{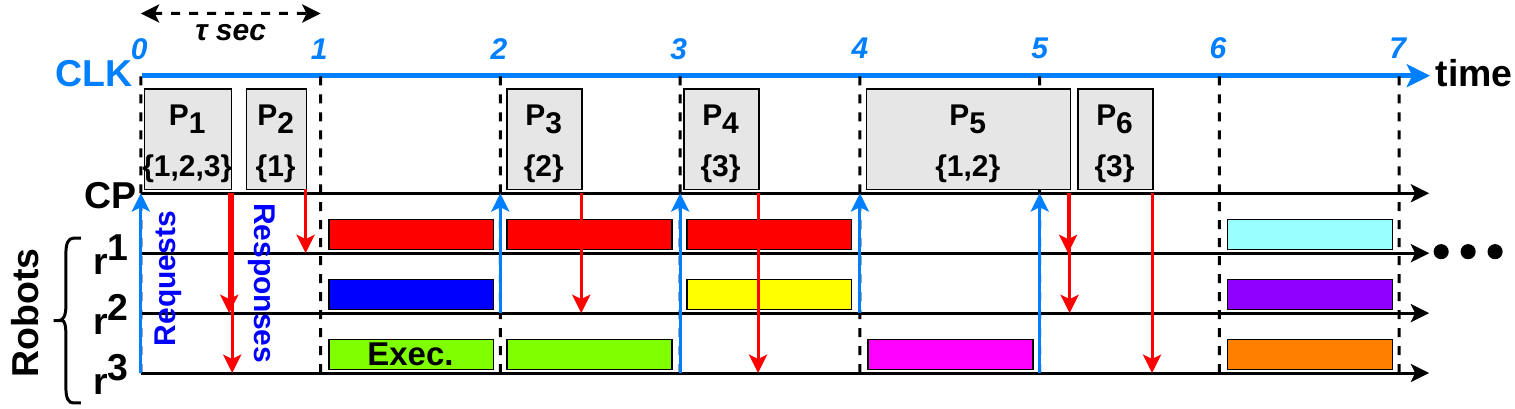}
    \caption{Overview of concurrent CPP}
    \label{fig:concpp}
\end{figure}

\begin{example}
\label{exp:concpp}
In Figure~\ref{fig:concpp}, we show a schematic diagram of the overall concurrent CPP framework for three robots in an arbitrary workspace. 
In the initial planning round P\textsubscript{1}, the CP finds paths timestamped with $1$ for participants $r^2$ and $r^3$ only. 
So, it begins P\textsubscript{2} for the inactive participant $r^1$ of P\textsubscript{1}, whose path gets timestamped with $1$ as well. 
When the CP replans $r^2$ in P\textsubscript{3}, non-participants $r^1$ and $r^3$ follow their existing paths in parallel. 
The new path of $r^2$ (observe the new color) gets timestamped with $3$. 
Similarly, it replans $r^3$ in P\textsubscript{4}, whose path gets timestamped with $4$. 
However, in P\textsubscript{5}, by the time it ends replanning $r^1$ and $r^2$, whose paths get timestamped with $6$, it receives a new request from $r^3$. 
So, it replans $r^3$ in P\textsubscript{6}, whose path gets timestamped with $6$ also. 
\end{example}

\input{4a_ds.tex}
\input{4b_concpp.tex}
\input{4c_concpp_round.tex}
\input{4d_participants}

%% file: 4a_ds.tex
\subsection{Data structures at CP}
\label{subsec:concpp_ds}

Let $I_{par} \subseteq [R]$ be the set of IDs of the robots participating in the coming planning round, and $S$ be the set of their corresponding states. 
Also, let $\eta \in [R]_0$ be the number of participants needed to begin the coming planning round, which the CP sets \textit{dynamically} after each planning round to replan for inactive participants or new participants. 
By $T_{stop}$, we denote an array of size $R$, where $T_{stop}[i] \in \mathbb{N}_0$ denotes the $CLK$ value when $r^i$ finishes following its latest path. 

%% file: 4b_concpp.tex
\subsection{Overall Concurrent Coverage Path Planning}
\label{subsec:concpp}

\input{algorithms/robot}

In this section, we explain the overall concurrent CPP in detail with the help of Algorithms \ref{algo:robot} and \ref{algo:concpp}. 
First, we explain Algorithm \ref{algo:robot}, which runs at each robot $r^i$. 
At the beginning, $r^i$ initializes its local view $W^i=\langle W^i_u, W^i_o, W^i_g, W^i_c \rangle$ (line 1), where $W^i_u$, $W^i_o$, $W^i_g$, and $W^i_c$ are the set of \textit{unexplored} cells, \textit{obstacles} (explored cells that are obstacle-occupied), \textit{goals} (explored cells that are obstacle-free and yet to get visited), and \textit{covered} (explored cells that are obstacle-free and have already been visited) cells, respectively. 
Next, it creates a request message $\mathtt{M_{req}}$, containing its $\mathtt{id}\ i$, current $\mathtt{state}\ s^i_0$, and local $\mathtt{view}\ W^i$ (line 2). 
Then, it sends the $\mathtt{M_{req}}$ to the CP (line 4) and waits for the response message $\mathtt{M_{res}}$ to arrive (line 5). 
Upon arrival, it extracts its path $\sigma^i$ and associated timestamp $T_{start} \in \mathbb{N}$ from the $\mathtt{M_{res}}$, and starts following $\sigma^i$ from $T_{start}$ (line 6). 
While following, it uses its rangefinders to explore previously unexplored workspace cells and updates $W^i$ accordingly. 
Finally, it reaches its goal $\mathcal{L}(s^i_{|\sigma^i|})$ corresponding to the state $s^i_{|\sigma^i|}$ when $CLK$ becomes $T_{start} + |\sigma^i|$. 
So, it updates the $\mathtt{M_{req}}$ with $s^i_{|\sigma^i|}$ and $W^i$ (line 7) and sends it to the CP in a \textbf{while} loop (lines 3-7). 

\input{algorithms/concpp}

Now, we explain \FnConCPP (Algorithm \ref{algo:concpp}), which stands for \textbf{Con}current \textbf{CPP}. 
After initializing the required data structures (lines 1-4), the CP starts a \textit{service} (lines 5-9) to receive requests from the robots in a \textit{mutually exclusive} manner. 
When it receives a request $\mathtt{M_{req}}$ from $r^i$, it adds $r^i$'s ID $i$ to $I_{par}$ (line 6), current state $s^i_0$ to $S$ (line 7), and updates the global view $W$ with $r^i$'s local view $W^i$ (line~8) \cite{DBLP:conf/arxiv/MitraS23}. 
Thus, $r^i$ becomes a participant in the coming planning round. 
Then, it invokes \FnCheckCPPCriteria (line 9) to check whether it can start a planning round. 

In \FnCheckCPPCriteria (lines 10-23), the CP decides not only whether to \textit{start} or \textit{skip} a planning round but also whether to \textit{stop} itself based on the current information about the workspace and the robots. 
First, the CP checks whether there are $\eta$ participants, i.e., $|I_{par}| = \eta$ (line 11). 
If yes, it finds the goals $\widetilde{W_g}$ already assigned to the non-participants $[R] \setminus I_{par}$ (denoted by $\overline{I_{par}}$) in previous rounds (line 12). 
Formally, $\widetilde{W_g} = \{\mathcal{L}(s^i_{|\pi^i|})\ |\ i \in \overline{I_{par}} \land |\pi^i| > 0\}$, i.e., $\widetilde{W_g}$ contains the goals corresponding to the non-participants' last states in their full paths. 
Recall that the first state in a full path is the start state of the corresponding robot, which is not a goal by definition (line 3). 
As the CP cannot assign the participants to $\widetilde{W_g}$, it checks whether there are unassigned goals, i.e., $W_g \setminus \widetilde{W_g}$ left in the workspace (line 13). 
If yes, it starts a planning round (lines 14-19). 
So, it takes a snapshot of the current information (lines 14-16) and invokes \FnConCPPRound with the snapshot (line 19), which generates timestamped paths for the participants, leading to some unassigned goals. 
Notice that the CP invokes \FnConCPPRound using a \textit{thread} $\mathtt{THD}$ to simultaneously receive new requests from some non-participants when the current planning round goes on for the participants. 
So, it resets $I_{par}$ and $S$ (line 17) to receive new requests and subsequently $\eta$ (line 18) to prevent another planning round from starting until the current planning round ends. 
Otherwise, i.e., when there are no unassigned goals, the CP cannot plan for the participants. 
So, it checks whether all the robots are participants, i.e., $|I_{par}| = R$ (line 20). 
If yes, the CP stops (line 21) as it indicates complete coverage \longversion{(proved in Theorem \ref{theorem:concpp_complete_coverage})}. 
Otherwise, the CP skips planning (lines 22-23) due to the unavailability of goals. 
Nevertheless, new goals (if any) get added to $W_g$ when some non-participants send their requests after reaching their goals, thereby becoming participants. 
So, the CP needs to increment $\eta$ for those non-participants by invoking \FnIncrementEta (line 23). 

In \FnIncrementEta (lines 24-28), $T_{min} \in \mathbb{N}_0$ denotes the $CLK$ when at least one non-participant becomes a participant (line 25). 
So, the \textbf{for} loop (lines 26-28) increments $\eta$ by the number of non-participants that become participants at $T_{min}$, enabling the \textbf{if} condition (line 11) to get satisfied. 

\shortversion
{
\begin{figure}[t]
    \centering
    \includegraphics[scale=0.24]{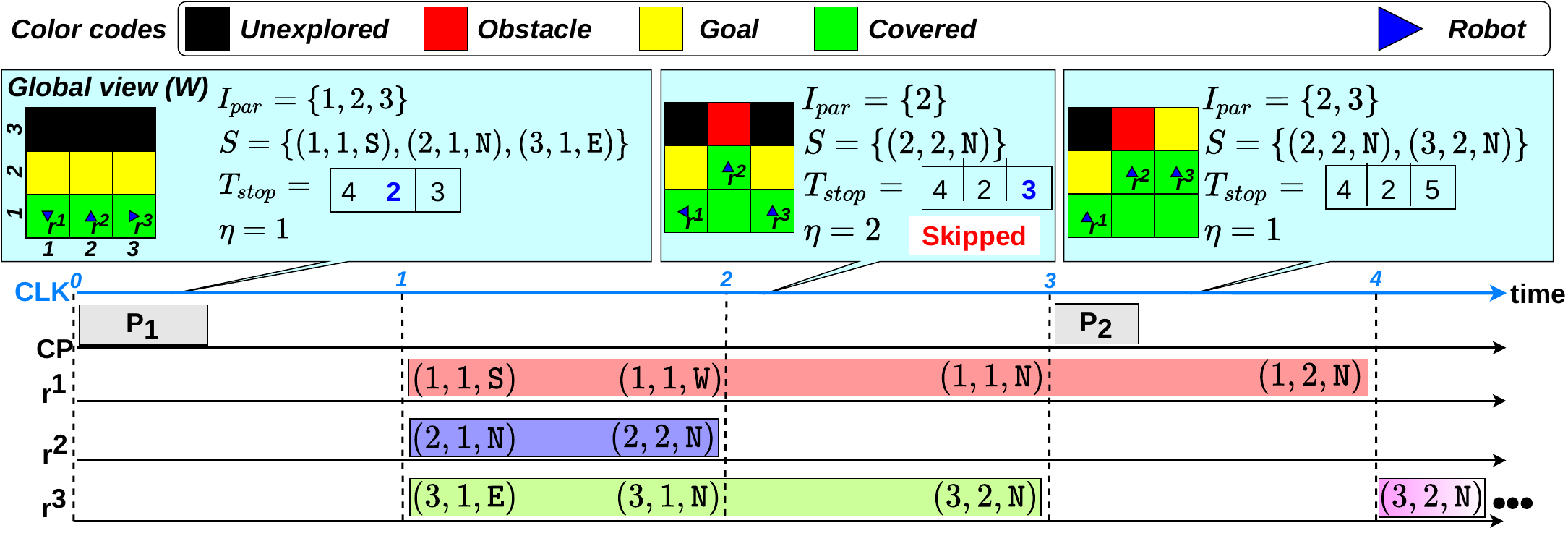}
    \caption{Skip replanning the participants}
    \label{fig:skip}
\end{figure}
}
\longversion
{
\begin{figure*}[t]
    \centering
    \includegraphics[scale=0.4]{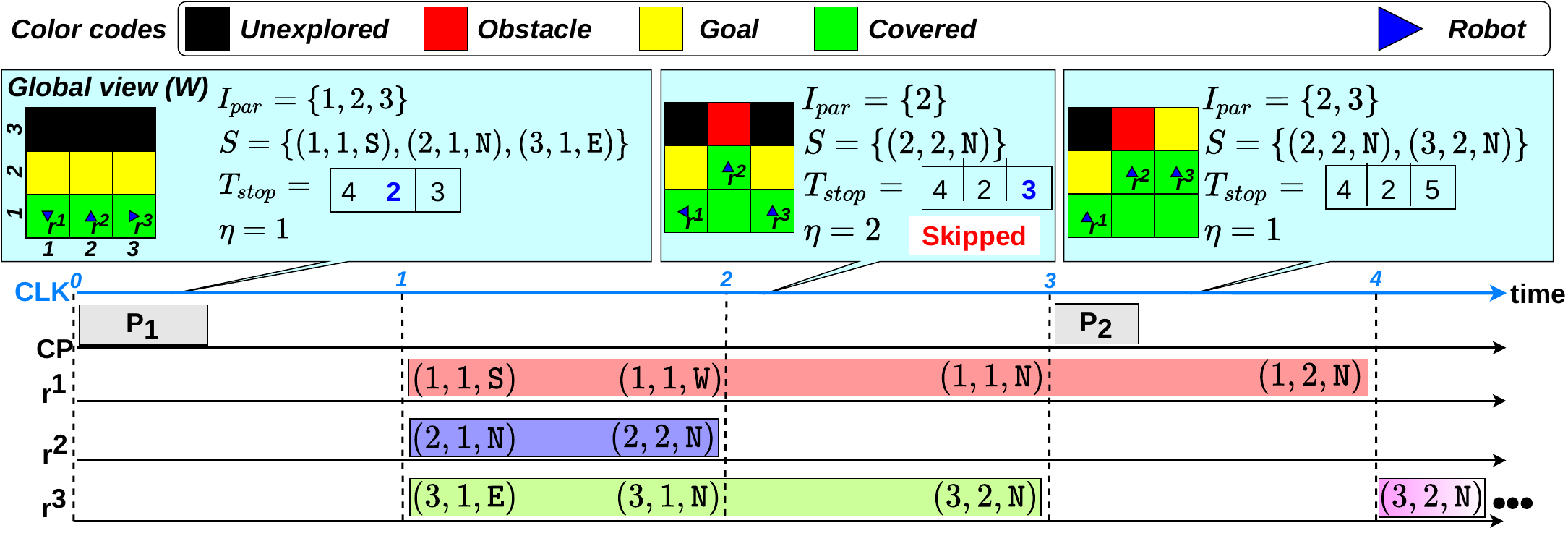}
    \caption{Skip replanning the participants}
    \label{fig:skip}
\end{figure*}
}

\begin{example}
\label{exp:skip}
In Figure \ref{fig:skip}, we show an example where the CP skips replanning a participant due to the unavailability of goals. 
Initially, all three robots $r^1, r^2$, and $r^3$ participate. 
As there are no non-participants, $\widetilde{W_g} = \emptyset$. 
Moreover, $W_g \setminus \widetilde{W_g} \neq \emptyset$. 
So, the first round P\textsubscript{1} begins, where the CP finds paths, timestamped with $1$, of lengths $3, 1,$ and $2$, respectively. 
So, it updates $T_{stop}$ accordingly and sets $\eta = 1$ because $r^2$ reaches its goal at $CLK = 2$, which is the earliest. 
When $r^2$ sends its updated request, the CP cannot start another round because the non-participants' reserved goals $\widetilde{W_g} = \{(1, 2), (3, 2)\}$ make $W_g \setminus \widetilde{W_g} = \emptyset$. 
So, the CP skips replanning $r^2$ (line 22 of Algorithm \ref{algo:concpp}) and increments $\eta$ to $2$ because $r^3$'s request arrives at $CLK = 3$, which is now the earliest. 
When $r^3$ sends its updated request, a goal becomes available as $\widetilde{W_g} = \{(1, 2)\}$ make $W_g \setminus \widetilde{W_g} \neq \emptyset$. 
So, P\textsubscript{2} begins for $r^2$ and $r^3$. 
\end{example}

%% file: algorithms/robot.tex
\begin{algorithm}[t]
    \footnotesize
    \DontPrintSemicolon
    
    \caption{$\mathtt{Robot}(i, s^i_0, X, Y)$}
    \label{algo:robot}
    
    $W^i \gets \mathtt{init\_localview}(s^i_0, X, Y)$\;
    $\mathtt{M_{req}[id, state, view]} \gets [i, s^i_0, W^i]$       \tcp*{Create $\mathtt{M_{req}}$}
    
    \While{true}{
        $\mathtt{send\_localview(M_{req})}$\;
        $\langle \sigma^i, T_{start}\rangle \gets \mathtt{receive\_path(M_{res})}$\;
        $W^i \gets \mathtt{follow\_path}(\sigma^i, T_{start})$\;
        $\mathtt{M_{req}[state, view]} \gets [s^i_{|\sigma^i|}, W^i]$       \tcp*{Update $\mathtt{M_{req}}$}
    }
\end{algorithm}

%% file: algorithms/concpp.tex
\begin{algorithm*}[t]
    \footnotesize
    \DontPrintSemicolon
    
    \caption{$\mathtt{ConCPP}(R, S)$}
    \label{algo:concpp}

    
    \begin{multicols}{2}
    
    $I_{par} \gets \emptyset, S \gets \emptyset, W \gets \emptyset$     \tcp*{Init.\ }
    $\eta \gets R$      \tcp*{Intended number of participants\ }
    $\Pi \gets \{\pi^i\ |\ i \in [R]\}$ where $\pi^i \gets s^i_0$\;
    $T_{stop}[i] \gets 0, \forall i \in [R]$\;
    \BlankLine
    \Srv{\SrvRcvLV{$\mathtt{M_{req}}$}}{
        $I_{par} \gets I_{par} \cup \{\mathtt{M_{req}.id}\}$\;
        $S \gets S \cup \{\mathtt{M_{req}.state}\}$\;
        $W \gets \mathtt{update\_globalview(M_{req}.view)}$\;
        \FnCheckCPPCriteria{}\;
    }
    \BlankLine    
    \Fn{\FnCheckCPPCriteria{}}{
        \If{$|I_{par}| = \eta$}{
            $\widetilde{W_g} \gets \mathtt{get\_reserved\_goals}(\overline{I_{par}}, \Pi)$\; 
            \If(\tcp*[f]{Start\ }){$(W_g \setminus \widetilde{W_g}) \neq \emptyset$}{
                $W^* \gets W$       \tcp*{Take snapshot\ }
                $I^*_{par} \gets I_{par}$\;
                $S^* \gets S$\;
                $I_{par} \gets \emptyset, S \gets \emptyset$       \tcp*{Reinit.\ }
                $\eta \gets 0$\;
                \textcolor{magenta}{$\mathtt{THD}$}.\FnConCPPRound{$W^*, \widetilde{W_g}, I^*_{par}, S^*$}\;
            }
            \ElseIf(\tcp*[f]{Stop (Coverage Complete)}){$|I_{par}| = R$}{
                \Exit()\;
            }
            \Else(\tcp*[f]{Skip\ }){
                \FnIncrementEta($\overline{I_{par}}$)\;
            }
        }
    }
    \BlankLine
    \Fn{\FnIncrementEta{$I$}}{
        $T_{min} \gets \min \{T_{stop}[i]\ |\ i \in I \land\ T_{stop}[i] > CLK\}$\;
        \For{$i \in I$}{
            \If{$T_{stop}[i] = T_{min}$}{
                $\eta \gets \eta + 1$     \tcp*{New participant to be \ \ }
            }
        }
    }
    \BlankLine
    \Fn{\FnConCPPRound{$W^*, \widetilde{W_g}, I^*_{par}, S^*$}}{
        $\langle \Sigma, T_{start} \rangle \gets \FnConCPPForPar(W^*, \widetilde{W_g}, I^*_{par}, S^*, M, R, \Pi)$\;
        \For{$i \in I^*_{par}$}{
            \If(\tcp*[f]{Active \ \ }){$|\sigma^i| > 0$}{
                $\zeta \gets \mathtt{dummy\_path}(s^i_0, T_{start} - T_{stop}[i] - 1)$\;
                $\pi^i \gets \pi^i : \zeta : \sigma^i$       \tcp*{Concatenation \ \ }
                $T_{stop}[i] \gets T_{start} + |\sigma^i|$\;
            }
            \Else(\tcp*[f]{Inactive \ \ }){
                $I_{par} \gets I_{par} \cup \{i\}$\;
                $S \gets S \cup \{s^i_0\}$\;
                $\eta \gets \eta + 1$\;
            }
        }
        \For{$j \in \overline{I^*_{par}}$}{
            \If{$T_{stop}[j] \leq CLK$}{
                $\eta \gets \eta + 1$       \tcp*{New participant \ \ }
            }
        }
        \If{$\eta = 0$}{
            \FnIncrementEta($[R]$)\;
        }
        $\mathtt{send\_paths\_to\_active\_participants}$($I^*_{par}, \Sigma, T_{start}$)\;
        \FnCheckCPPCriteria{}\;
    }
    \end{multicols}
     \BlankLine
\end{algorithm*}

%% file: 4c_concpp_round.tex
\subsection{Coverage Path Planning in the current round}
\label{subsec:concpp_round}

\FnConCPPRound (lines 29-46) performs two tasks, viz., replanning the participants of the current round (line 30) and dynamically setting $\eta$ for the next round (lines 31-44). 
First, the CP invokes \FnConCPPForPar (Algorithm \ref{algo:cpp_for_par}, discussed in section \ref{subsec:cpp_for_participants}), which finds the paths $\Sigma = \{\sigma^i\ |\ i \in I^*_{par}\}$ timestamped with $T_{start}$ for the participants $I^*_{par}$. 
These paths $\Sigma$ must be collision-free w.r.t. the existing paths of the non-participants $\overline{I^*_{par}}$, which the CP checks from $\Pi$. 

Next, the CP examines $\Sigma$ (lines 31-39) to determine which participants are active (lines 32-35) and which are not (lines 36-39). 
If a participant $r^i$, where $i \in I^*_{par}$, is found active (i.e., $|\sigma^i| > 0$), the CP creates a \textit{dummy} path $\zeta$, containing $T_{start} - T_{stop}[i] - 1$ many $s^i_0$ state (line 33) because $r^i$ \textit{halts} at $s^i_0$ for $CLK \in \{(T_{stop}[i] + 1), \cdots ,(T_{start} - 1)\}$ before starting to follow $\sigma^i$ from $T_{start}$. 
Then, it updates $r^i$'s full path $\pi^i$ with $\zeta$ and $\sigma^i$ (line 34). 
It also updates $T_{stop}[i]$ (line 35) because after following $\sigma^i$ from $T_{start}$, $r^i$ reaches its goal at $T_{start} + |\sigma^i|$ and sends its updated request. 
In contrast, if $r^i$ is found inactive, the CP adds $r^i$'s ID $i$ to $I_{par}$, its state $s^i_0$ to $S$, and increments $\eta$ by $1$ (lines 37-39) as it reattempts to find a path for $r^i$ in the next round. 
Thus, inactive participants in the current round will become participants in the next round. 

As the current round goes on, a non-participant $r^j$, where $j \in \overline{I^*_{par}}$, sends its request to the CP if it reaches its goal recently, i.e., $T_{stop}[j] \leq CLK$. 
So, the CP increases $\eta$ for planning their paths in the next round (lines 40-42). 
At this point (line 43), $\eta$ remains $0$ if all the participants $I^*_{par}$ are found active (lines 32-35) and all the non-participants $\overline{I^*_{par}}$ (if any) are yet to reach their goals (the implicit \textbf{else} part of the \textbf{if} statement in line 41). 
So, the CP invokes \FnIncrementEta for all the robots (lines 44 and 24-28), which first determines $T_{min}$ (line 25), the earliest $CLK$ value when it will receive new requests, and subsequently sets $\eta$ to the number of requests arriving at $T_{min}$ (lines 26-28). 
In the penultimate step, the CP invokes $\mathtt{send\_paths\_to\_active\_participants}$ (line 45) that creates a response message $\mathtt{M_{res}}$ containing path $\sigma^i$ and timestamp $T_{start}$ for each active participant $r^i$, and subsequently sends that $\mathtt{M_{res}}$ to $r^i$. 

Finally, the CP concludes the current round by invoking \FnCheckCPPCriteria (line 46), which checks whether it can start the next round to plan paths for the inactive participants of the current round (lines 36-39), the new participants (lines 40-42), or both. 
Lines 31-46 constitute a critical section as the CP updates $I_{par}, S$, and $\eta$. 

\begin{example}
In Example \ref{exp:concpp}, the CP calls \FnIncrementEta (line 44 of Algorithm \ref{algo:concpp}) in $P\textsubscript{2-4, 6}$ but not in $P\textsubscript{1}$ and $P\textsubscript{5}$. 
In $P\textsubscript{1}$, the CP finds participant $r^1$ inactive (line 39) and during $P\textsubscript{5}$ it receives a request from non-participant $r^3$ (line 42). 
\end{example}

%% file: 4d_participants.tex
\subsection{Coverage Path Planning for the Participants}
\label{subsec:cpp_for_participants}

\FnConCPPForPar (Algorithm \ref{algo:cpp_for_par}) finds the participants' paths $\Sigma$ with timestamp  $T_{start}$ in two steps, viz., \textit{cost-optimal path} finding (line 1) and \textit{collision-free path} finding (lines 2-11). 
It is based on Algorithm 3 of \cite{DBLP:conf/arxiv/MitraS23}, which finds only $\Sigma$ while keeping the non-participants' existing paths \textbf{intact}. 

\input{algorithms/concppforpar}

\smallskip
\subsubsection{Cost-optimal Paths}

In the current round, let $R^*$ be the number of participants, i.e., $R^* = |I^*_{par}| = |S^*|$ and $G^*$ be the number of unassigned goals, i.e., \mbox{$G^* = |W^*_g \setminus \widetilde{W_g}|$}. 
So, the CP finds an assignment $\Gamma$ using the Hungarian algorithm~\cite{kuhn1955hungarian} to cost-optimally assign the participants $I^*_{par}$ to the unassigned goals having IDs $[G^*]$. 
Note that the CP uses the A$^*$ algorithm \cite{DBLP:journals/tssc/HartNR68} to find the optimal path and its cost for each participant-unassigned goal pair. 
Formally, \mbox{$\Gamma: I^*_{par} \rightarrow [G^*] \cup \{\mathsf{NULL}\}$}, where $\Gamma[i]$ denotes the goal ID assigned to a participant $r^i$. 
Note that $\Gamma[i]$ can be $\mathsf{NULL}$, e.g., when $R^* > G^*$. 
Such a participant is said to be \emph{inactive}, whose optimal path $\varphi^i$ only contains its current state $s^i_0$. 
Otherwise, it is said to be \emph{active}. 
Therefore, its optimal path $\varphi^i$ leads $r^i$ from its current state $s^i_0$ to the goal having id $\Gamma[i]$ using motions $M$. 
Such a path $\varphi^i$ passes through some cells in $W^*_c \cup W^*_g$ but avoids cells in $W^*_o \cup W^*_u$. 
We denote the set of participants' optimal paths by $\Phi = \{\varphi^i\ |\ i \in I^*_{par}\}$. 

\smallskip
\subsubsection{Collision-free Paths}
\label{subsubsec:cfp_for_par}

The participants' optimal paths in $\Phi$ are not necessarily \textit{collision-free}, meaning a participant $r^i$, where $i \in I^*_{par}$, may collide with another participant $r^j$, $j(\neq i) \in I^*_{par}$ or a non-participant $r^k$, where $k \in \overline{I^*_{par}}$. 
So, the CP must make the participants' paths collision-free. 
Moreover, it must find the timestamp $T_{start}$ of those collision-free paths. 
So, first, the CP sets $T_{start}$ (line 5) to be the next $CLK$ value using a \textit{look-ahead} $la \in \mathbb{N}$ (line 2). 

Next, the CP precisely computes the remaining paths $\Sigma_{rem} = \{\sigma^k_{rem}\ |\ k \in \overline{I^*_{par}}\}$ of the non-participants starting from $T_{start}$ (line 6) using the following Equation \ref{eq:rem_path}. 

\begin{equation}
    \label{eq:rem_path}
    \begin{aligned}
        \sigma^k_{rem} = 
            \begin{cases}
                s^k_{|\pi^k|},  & \text{if $|\pi^k| \leq T_{start}$}\\
                s^k_{T_{start}}\ \cdots\ s^k_{|\pi^k|},  & \text{otherwise.}
            \end{cases}
    \end{aligned}
\end{equation}

The remaining path $\sigma^k_{rem}$ of a non-participant $r^k$, where $k \in \overline{I^*_{par}}$, contains the last state of $r^k$ in its full path $\pi^k$ if $r^k$ reaches its goal by $T_{start}$. 
Otherwise, $\sigma^k_{rem}$ contains the remaining part of $\pi^k$, which $r^k$ is yet to traverse from $T_{start}$. 

Now, the CP invokes $\mathtt{CFPForPar}$, which uses the idea of \textit{Prioritized Planning} to generate the collision-free paths $\Sigma$ for the participants $I^*_{par}$ without altering the remaining paths $\Sigma_{rem}$ for the non-participants $\overline{I^*_{par}}$. 
It is a two-step procedure. 
In the first step, it \textit{dynamically} prioritizes the participants based on their movement constraints in $\Phi$. 
For example, if the start location $\mathcal{L}(s^i_0)$ of a participant $r^i$ is on the path $\varphi^j$ of another participant $r^j$, then $r^i$ must depart from its start location before $r^j$ gets past that location. 
Similarly, if the goal location $\mathcal{L}(s^i_{|\varphi^i|})$ of $r^i$ is on $\varphi^j$ of $r^j$, then $r^j$ must get past that goal location before $r^i$ arrives at that location. 
The non-participants implicitly have higher priorities than the participants, as the CP does not change $\Sigma_{rem}$. 
In the last step, it \textit{offsets} the participants' paths by prefixing them with $\mathtt{H}$ moves in order of their priority to avoid collision with higher priority robots. 
When a collision between a participant and a non-participant becomes inevitable, it inactivates the participant and returns to the first step. 
Otherwise, the offsetting step succeeds, and the participants' paths $\Sigma$ become collision-free. 
In the worst case, all the participants may get inactivated to avoid collisions with the non-participants during offsetting (see Example \ref{exp:cfp_for_par}). 

\shortversion
{
\begin{figure}[t]
    \centering
    \includegraphics[scale=0.22]{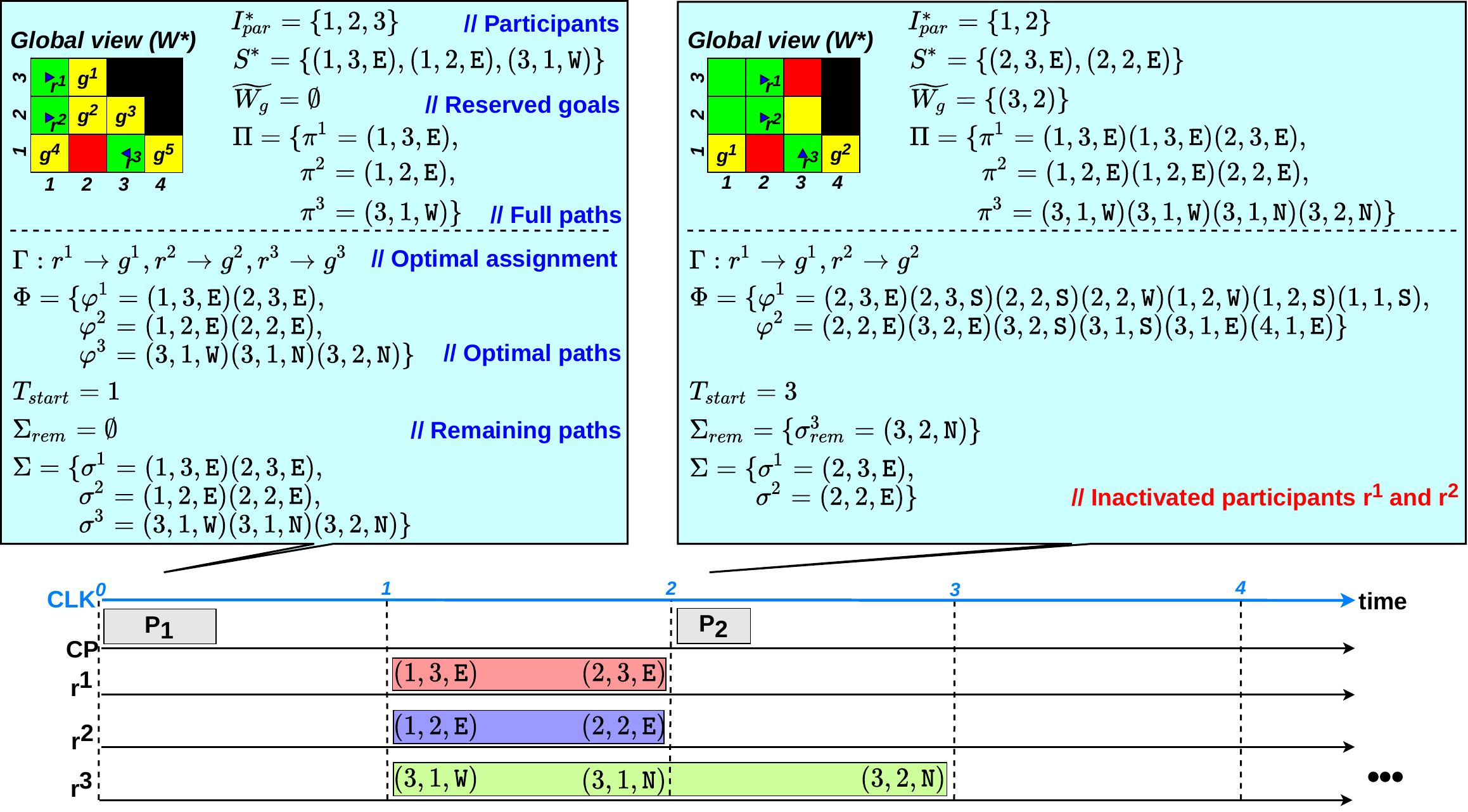}        
    \caption{Generating collision-free paths}
    \label{fig:cfp}
\end{figure}
}
\longversion
{
\begin{figure*}[t]
    \centering
    \includegraphics[scale=0.4]{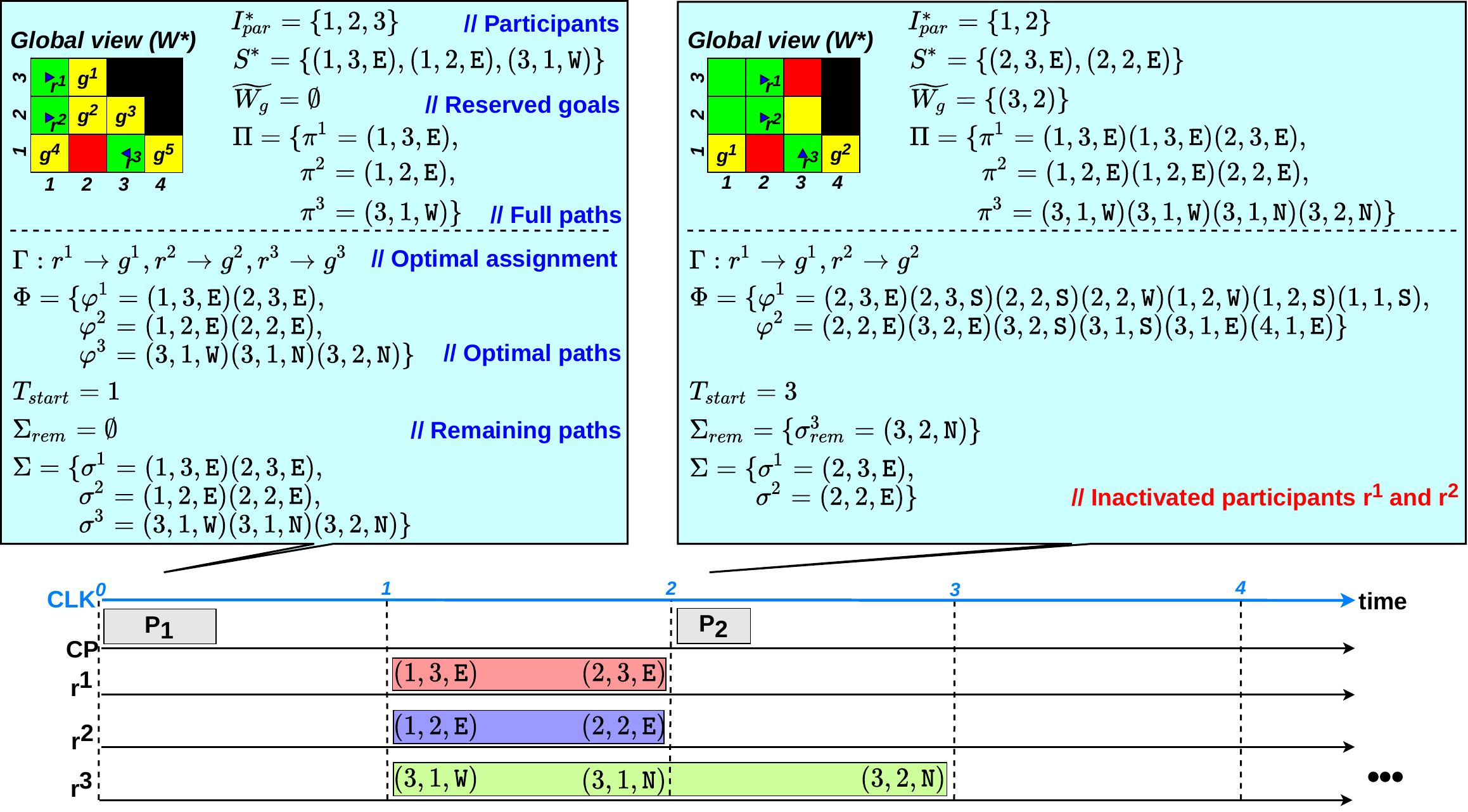}        
    \caption{Generating collision-free paths}
    \label{fig:cfp}
\end{figure*}
}
\begin{example}
\label{exp:cfp_for_par}
Figure \ref{fig:cfp} shows an example where only robots $r^1$ and $r^2$ participate in the planning round P\textsubscript{2}. 
So, the CP finds their optimal paths $\varphi^1$ and $\varphi^2$, leading them to goals $g^1$ and $g^2$, respectively. 
Participant $r^2$ has higher priority than $r^1$ as $r^2$'s start location is on $\varphi^1$, meaning $r^2$ must leave its start location before $r^1$ passes through that location. 
During offsetting, first, the CP inactivates $r^2$ because non-participant $r^3$'s remaining path $\sigma^3_{rem}$ blocks $\varphi^2$. 
Subsequently, the CP inactivates $r^1$ because recently inactivated $r^2$ now blocks $\varphi^1$. 
\end{example}

As the CP timestamps $\Sigma$ with $T_{start}$, an active participant $r^i$ must start following its path $\sigma^i$ from $T_{start}$. 
Put differently, $r^i$ must receive $\sigma^i$ from the CP before $CLK$ becomes $T_{start}$ (line 8). 
Otherwise, the CP \textit{reactively} updates $la$ (lines 4 and 9-11) to reattempt with a new $T_{start}$ (line 5). 

\begin{figure}
    \centering
    \includegraphics[scale=0.45]{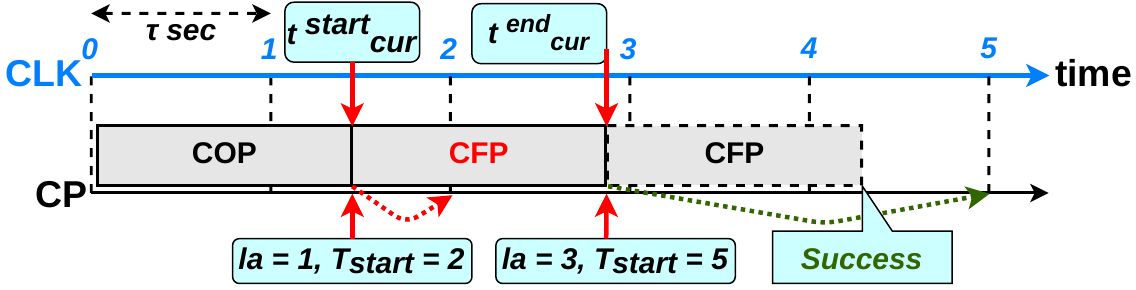}
    \caption{Reattempt to find the participants' collision-free paths}
    \label{fig:cfp_reattempt}
\end{figure}

\begin{example}
In Figure \ref{fig:cfp_reattempt}, we show an example where the CP sets $T_{start}$ to $2$ at the beginning of the collision-free pathfinding. 
As it finishes the computation, $CLK$ is already $2$. 
Hence, the participants' paths $\Sigma$ become space-time inconsistent, which the CP cannot send to the participants. 
So, it updates $la$ to $3$, sets $T_{start}$ to $5$, and reattempts. 
Now, the CP succeeds in finishing the computation on time. 
\end{example}

%% file: algorithms/concppforpar.tex
\begin{algorithm}[t!]
    \footnotesize
    \DontPrintSemicolon
    
    \caption{$\mathtt{ConCPPForPar}(W^*, \widetilde{W_g}, I^*_{par}, S^*, M, R, \Pi)$}
    \label{algo:cpp_for_par}

    \KwResult{Participants' paths $\Sigma$ with timestamp $T_{start}$}
    
    $\langle \Gamma, \Phi \rangle \gets \mathtt{COPForPar}(W^*, \widetilde{W_g}, I^*_{par}, S^*, M)$\;
    $la \gets 1$        \tcp*{Number of $CLK$ values to look-ahead}
    \While{$true$}{
        $t^{start}_{cur} \gets \mathtt{read\_current\_time()}$\;
        $T_{start} \gets CLK + la$\;
        $\Sigma_{rem} \gets \mathtt{get\_remaining\_paths}(\overline{I^*_{par}}, \Pi, T_{start})$\;
        $\Sigma \gets \mathtt{CFPForPar}(\Gamma, \Phi, I^*_{par}, R, \Sigma_{rem})$\;
        \lIf(\tcp*[f]{Success}){$CLK < T_{start}$}{
            \Break
        }
        \Else(\tcp*[f]{Failure! So, Reattempt}){
            $t^{end}_{cur} \gets \mathtt{read\_current\_time()}$\;
            $la \gets 1 + [\lfloor\{t^{end}_{cur} + (t^{end}_{cur} - t^{start}_{cur})\} / \tau\rfloor - CLK]$\;
        }
    }
\end{algorithm}

%% file: 5_theoretical_analysis.tex
\section{Theoretical Analysis}
\label{sec:theoretical_analysis}

In this section, first, we formally prove that \FnConCPP guarantees complete coverage of the unknown workspace $W$. 
Then, we analyze its time complexity. 

\input{5a_complete_coverage}
\input{5b_tc}

%% file: 5a_complete_coverage.tex
\longversion
{
\subsection{Proof of Complete Coverage}
\label{subsec:complete_coverage}
}
\longversion
{
First, we give the outline of the proof. 
Horizon-based CP \cite{DBLP:conf/iros/MitraS22}, in each horizon, replans all the robots and guarantees that at least one robot remains active to visit its goal, which we mention in Lemma \ref{lem:gamrcpp_horizon}. 
Another horizon-based CP~\cite{DBLP:conf/arxiv/MitraS23}, in each horizon, replans for only a subset of robots (the participants) on demand. 
Therefore, in a horizon, if all the robots are participants, i.e., $R^* = R$, the latter idea becomes equivalent to the former, which we acknowledge in Lemma \ref{lem:cppforpar_equivalence_gamrcpphorizon}. 
As our CP builds upon the on-demand replanning approach of \cite{DBLP:conf/arxiv/MitraS23}, in Lemma \ref{lem:concppforpar_equivalence_cppforpar}, we establish the fact that in a round, if all the robots are participants, there is at least one active participant to visit its goal. 
To prove incremental coverage, in Lemma \ref{lem:eta_R}, we establish that rounds where all the robots are participants occur always eventually. 
Finally, in Theorem \ref{theorem:concpp_complete_coverage}, we prove that the incremental coverage eventually leads to complete coverage of $W$ when \FnConCPP terminates. 
}
\longversion
{
\begin{lemma}[Lemma $4.3$ in \cite{DBLP:conf/iros/MitraS22}]
\label{lem:gamrcpp_horizon}
\FnGAMRCPPHorizon ensures that at least one goal gets visited in each horizon. 
\end{lemma}
}
\longversion
{
\begin{lemma}[Lemma $2$ in \cite{DBLP:conf/arxiv/MitraS23}]
\label{lem:cppforpar_equivalence_gamrcpphorizon}
\FnCPPForPar $\equiv$ \FnGAMRCPPHorizon if \mbox{$R^* = R$}. 
\end{lemma}
}
\longversion
{
\begin{lemma}
\label{lem:concppforpar_equivalence_cppforpar}
In a round, \FnConCPPForPar $\equiv$ \FnCPPForPar if all the robots are participants. 
\end{lemma}
}
\longversion
{
\begin{proof}
The difference between \FnConCPPForPar (Algorithm \ref{algo:cpp_for_par}) and \FnCPPForPar (Algorithm 3 in \cite{DBLP:conf/arxiv/MitraS23}) is that in \FnConCPPForPar, \FnCFPForPar (line 7) is in a \textbf{while} loop (lines 3-11). 
In a round, no non-participants exist if all the robots are participants, i.e., $R^* = R$. 
So, the non-participants' remaining paths $\Sigma_{rem} = \emptyset$ (line 6). 
As explained in Section \ref{subsubsec:cfp_for_par}, this \textbf{while} loop iterates \textit{at most} twice to find the participants' collision-free paths $\Sigma$, additionally the timestamp $T_{start}$ (lines 7-8). 
Thus, \FnConCPPForPar outputs the same $\Sigma$ as \FnCPPForPar does if $R^* = R$. 
Hence, \FnConCPPForPar $\equiv$ \FnCPPForPar if all the robots are participants. 
\end{proof}
}
\longversion
{
\begin{lemma}
\label{lem:eta_R}
Always eventually, there will be a round with all the robots as participants. 
\end{lemma}
}
\longversion
{
\begin{proof}
Initially, the number of intended participants $\eta = R$ (line 2 in Algorithm \ref{algo:concpp}). 
So, in the initial round, all the robots are participants, i.e., $R^* = R$ trivially (line 11). 
Before the CP invokes \FnConCPPRound (line 19), the CP resets $\eta$ (line 18) to reinitialize it at the end of the planning round for replanning inactive participants in the next planning round (line 39). 
In the worst case, by Lemma \ref{lem:concppforpar_equivalence_cppforpar}, \ref{lem:cppforpar_equivalence_gamrcpphorizon}, and \ref{lem:gamrcpp_horizon}, \FnConCPPForPar finds exactly one participant active in $\Sigma$ (line 30), thereby reinitializing $\eta = R - 1$ (line 39) for replanning $R - 1$ inactive participants in the next planning round (lines 46 and 11). 
Moreover, in the worst case, due to the unavailability of unassigned goals, i.e., $W_g \setminus \widetilde{W_g} = \emptyset$ (line 13), the CP skips replanning $R - 1$ participants, invoking \FnIncrementEta (line 23) and making $\eta = R$ again, thereby, the only active participant of the previous round also becomes a participant after reaching its goal and sending its updated $\mathtt{M_{req}}$ (lines 5-9). 
Thus, all the robots become participants again, making $R^* = R$ (line 11). 
Hence, always eventually, there will be a round with all the robots as participants. 
\end{proof}
}
\begin{theorem}
\label{theorem:concpp_complete_coverage}
\FnConCPP eventually stops, and when it stops, it ensures complete coverage of $W$. 
\end{theorem}
\shortversion
{
\begin{proof}
Please refer to the longer version of this paper~\cite{DBLP:conf/arxiv/MitraS24} for the proof. 
\end{proof}
}
\longversion
{
\begin{proof}
All the robots are initially participants, i.e., $R^* = R$ (lines 2 and 11 in Algorithm \ref{algo:concpp}). 
So, there are no non-participants, thereby their assigned goals $\widetilde{W_g} = \emptyset$ (line 12). 
Now, if there are no goals, 
i.e., $W_g = \emptyset$, it falsifies the \textbf{if} condition in line 13 and satisfies the \textbf{if} condition in line 20; hence the CP stops. 
Otherwise, the CP invokes \FnConCPPRound (line 19), which by Lemma \ref{lem:concppforpar_equivalence_cppforpar}, \ref{lem:cppforpar_equivalence_gamrcpphorizon}, and \ref{lem:gamrcpp_horizon}, finds at least one participant active in $\Sigma$ (line 30). 
So, while following its path, an active participant explores the unexplored cells $W_u$ into either obstacles $W_o$ or goals $W_g$ (if any) as $W_{free}$ is
connected, eventually covering its goal into $W_c$. 
Thus, $W_c$ increases. 
By Lemma \ref{lem:eta_R}, always eventually, this planning round keeps occurring, where all the robots are participants, meaning $R^* = R$ and $\widetilde{W_g} = \emptyset$. 
Now, if $W_g = \emptyset$, the CP stops, which entails $W_c = W_{free}$. 
Otherwise, the CP invokes \FnConCPPRound. 
\end{proof}
}

%% file: 5b_tc.tex
\longversion
{
\subsection{Time Complexity Analysis}
\label{subsec:time_complexity}
}

\longversion
{
We now analyze the time complexity of \FnConCPP. 
\begin{lemma}[Lemma $4$ of \cite{DBLP:conf/arxiv/MitraS23}]
\label{lem:tc_cfpforpar}
The function $\mathtt{CFPForPar}$ takes $\mathcal{O}(R^3)$. 
\end{lemma}
}
\longversion
{
\begin{lemma}
\label{lem:tc_copforpar}
The function $\mathtt{COPForPar}$ takes $\mathcal{O}(|W|^3)$. 
\end{lemma}
}
\longversion
{
\begin{proof}
   Please refer to Lemma $5$ of \cite{DBLP:conf/arxiv/MitraS23}. 
\end{proof}
}
\longversion
{
\begin{lemma}
\label{lem:tc_concppforpar}
The function \FnConCPPForPar takes $\mathcal{O}(|W|^3)$. 
\end{lemma}
}
\longversion
{
\begin{proof}
Finding the remaining paths $\Sigma_{rem}$ of the non-participants $\overline{I^*_{par}}$ from $\Pi$ (line 6 in Algorithm \ref{algo:cpp_for_par}) takes \mbox{$\mathcal{O}(R - R^*)$}. 
By Lemma \ref{lem:tc_cfpforpar}, $\mathtt{CFPForPar}$ (line 7) takes $\mathcal{O}(R^3)$. 
The rest of the body of the \textbf{while} loop (lines 3-11) takes $\mathcal{O}(1)$. 
Furthermore, this loop iterates at most twice, as explained in Section \ref{subsubsec:cfp_for_par}. 
So, this loop takes total $\mathcal{O}(2 \cdot ((R -R^*) + R^3 + 1))$, which is $\mathcal{O}(R^3)$ as $R^* = \mathcal{O}(R)$. 
Now, by Lemma \ref{lem:tc_copforpar}, $\mathtt{COPForPar}$ (line 1) takes $\mathcal{O}(|W|^3)$. 
Also, initialization of $la$ (line 2) takes $\mathcal{O}(1)$. 
Hence, \FnConCPPForPar total takes $\mathcal{O}(|W|^3 + 1 + R^3)$, which is $\mathcal{O}(|W|^3)$ as $R = \mathcal{O}(|W|)$. 
\end{proof}
}
\longversion
{
\begin{lemma}
\label{lem:tc_increment_eta}
The function \FnIncrementEta takes $\mathcal{O}(R)$. 
\end{lemma}
}
\longversion
{
\begin{proof}
The body of the \textbf{for} loop (lines 26-28 in Algorithm \ref{algo:concpp}) takes $\mathcal{O}(1)$ and the loop iterates at most $|I|$ times. 
So, this loop total takes $\mathcal{O}(|I|)$. 
Similarly, determining $T_{min}$ (line 25) takes $\mathcal{O}(|I|)$. 
So, \FnIncrementEta total takes $\mathcal{O}(|I|)$, which is $\mathcal{O}(R)$ as $|I| = \mathcal{O}(R)$. 
\end{proof}
}
\longversion
{
\begin{lemma}
\label{lem:tc_check_cpp_criteria}
The function \FnCheckCPPCriteria takes $\mathcal{O}(|W|)$. 
\end{lemma}
}
\longversion
{
\begin{proof}
By Lemma \ref{lem:tc_increment_eta}, \FnIncrementEta (line 23 in Algorithm \ref{algo:concpp}) takes $\mathcal{O}(R)$. 
Likewise, checking whether all the robots are participants (line 20) takes $\mathcal{O}(R^*)$. 
So, the \textbf{if-else} block (lines 20-23) takes $\mathcal{O}(R^* + (1 + R))$, which is $\mathcal{O}(R)$ as $R^* = \mathcal{O}(R)$. 
Now, taking a snapshot of $W$ (line 14) takes $\mathcal{O}(|W|)$. 
Also, taking snapshots of $I_{par}$ and $S$ (lines 15-16) take $\mathcal{O}(R^*)$ each. 
The rest of the body of the \textbf{if} block (lines 14-19) takes $\mathcal{O}(1)$. 
So, the body of this block total takes $\mathcal{O}(|W| + R^* + 1)$. 
Next, checking whether there exists any unassigned goal (line 13) takes $\mathcal{O}(|W|)$. 
Therefore, the \textbf{if-else if-else} block (lines 13-23) total takes $\mathcal{O}(|W| + ((|W| + R^* + 1) + R))$, which is $\mathcal{O}(|W| + R)$. 
Next, computing the reserved goals $\widetilde{W_g}$ for the non-participants $\overline{I_{par}}$ from $\Pi$ (line 12) takes $\mathcal{O}(R - R^*)$ and checking whether there are $\eta$ participants (line 11) takes $\mathcal{O}(R^*)$. 
Hence, \FnCheckCPPCriteria total takes $\mathcal{O}(R^* + ((R - R^*) + (|W| + R)))$, which is $\mathcal{O}(|W|)$. 
\end{proof}
}
\longversion
{
\begin{lemma}
\label{lem:tc_concpp_round}
The function \FnConCPPRound takes $\mathcal{O}(|W|^3)$. 
\end{lemma}
}
\longversion
{
\begin{proof}
By Lemma~\ref{lem:tc_concppforpar}, \FnConCPPForPar (line 30) takes $\mathcal{O}(|W|^3)$. 
The body of the first \textbf{for} loop (lines 31-39) takes $\mathcal{O}(1)$, and it iterates $R^*$ times. 
So, this loop total takes $\mathcal{O}(R^*)$. 
Similarly, the body of the second \textbf{for} loop (lines 40-42) takes $\mathcal{O}(1)$, and it iterates $R - R^*$ times. 
So, this loop total takes $\mathcal{O}(R - R^*)$. 
By Lemma \ref{lem:tc_increment_eta}, \FnIncrementEta (line 44) takes $\mathcal{O}(R)$. 
So, the \textbf{if} block (lines 43-44) takes $\mathcal{O}(R)$. 
Next, \textit{parallelly} sending timestamped paths to the active participants (line 45) takes $\mathcal{O}(1)$. 
Finally, by Lemma \ref{lem:tc_check_cpp_criteria}, \FnCheckCPPCriteria (line 46) takes $\mathcal{O}(|W|)$. 
Hence, \FnConCPPRound total takes $\mathcal{O}(|W|^3 + R^* + (R - R^*) + R + 1 + |W|)$, which is $\mathcal{O}(|W|^3)$. 
\end{proof}
}
\longversion
{
\begin{lemma}
\label{lem:tc_rcv_lv}
The function \SrvRcvLV takes $\mathcal{O}(|W|^2)$. 
\end{lemma}
}
\longversion
{
\begin{proof}
Adding a participant's $\mathtt{id}$ to $I_{par}$ (line 6 in Algorithm \ref{algo:concpp}) and $\mathtt{state}$ to $S$ (line 7) take $\mathcal{O}(1)$ each. 
Also, updating the global view $W$ with the participant's local view $W^i$ (line 8) takes $\mathcal{O}(|W|^2)$ (please refer to Theorem $2$ of \cite{DBLP:conf/arxiv/MitraS23}). 
By Lemma \ref{lem:tc_check_cpp_criteria}, \FnCheckCPPCriteria (line 9) takes $\mathcal{O}(|W|)$. 
Hence, \SrvRcvLV total takes $\mathcal{O}(1 + |W|^2 + |W|)$, which is $\mathcal{O}(|W|^2)$. 
\end{proof}
}
\begin{theorem}
\label{lem:tc_concpp}
Time complexity of \FnConCPP is $\mathcal{O}(|W|^4)$. 
\end{theorem}

\begin{proof}
\shortversion
{
Please refer to the longer version of this paper~\cite{DBLP:conf/arxiv/MitraS24} for the proof. 
}
\longversion
{
The initialization block (lines 1-4 in Algorithm \ref{algo:concpp}) takes $\mathcal{O}(R)$. 
In the worst case, by Lemma \ref{lem:eta_R}, a planning round where all the robots participate keeps exactly one participant active and the rest inactive. 
So, that active participant follows its path, reaches its goal, and then sends its updated request to \SrvRcvLV (line 5), which by Lemma \ref{lem:tc_rcv_lv}, takes $\mathcal{O}(|W|^2)$. 
Additionally, \FnConCPPRound gets invoked (line 19), which by Lemma \ref{lem:tc_concpp_round}, takes $\mathcal{O}(|W|^3)$. 
Overall it needs $\mathcal{O}(|W_g|)$ such rounds, amounting to total $\mathcal{O}(R + (|W_g| \cdot (|W|^2 + |W|^3)))$, which is $\mathcal{O}(|W|^4)$ as $|W_g| = \mathcal{O}(|W|)$. 
}
\end{proof}

%% file: 6_evaluation.tex
\section{Evaluation}
\label{sec:evaluation}

\subsection{Implementation and Experimental Setup}
\label{subsec:implementation_and_experimental_setup}

We implement $\mathtt{Robot}$ (Algorithm \ref{algo:robot}) and \FnConCPP (Algorithm \ref{algo:concpp}) in a common ROS\longversion{\cite{key_ros}} 
C++ package\footnote{\textcolor{blue}{https://github.com/iitkcpslab/ConCPP}}, which runs in a workstation having Intel\textsuperscript{\textregistered} Xeon\textsuperscript{\textregistered} Gold 6226R Processor and 48 GB of RAM, and running Ubuntu $22.04.4$ LTS OS.

\subsubsection{Workspaces for experimentation}

We consider eight large $2$D grid benchmark workspaces of varying size and \textit{obstacle density} from \cite{stern2019mapf}.

\subsubsection{Robots for experimentation}

We consider multiple TurtleBots \cite{key_tb} (introduced before) to cover four of the workspaces and Quadcopters (introduced below) for the rest. 

\begin{example}
The state of a quadcopter is $s^i_j = (x, y) \in W$, which is its location in $W$. 
So, its set of motions is $M = \{$$\mathtt{Halt(H)}$, $\mathtt{MoveEast(ME)}$, $\mathtt{MoveNorth(MN)}$, $\mathtt{MoveWest(MW)}$, $\mathtt{MoveSouth(MS)}\}$, where $\mathtt{ME}$ moves it to the neighboring cell east of the current cell. 
Likewise, we define $\mathtt{MN, MW}$ and $\mathtt{MS}$. 
\end{example}

For each experiment, we \textit{incrementally} deploy $R \in \{128, 256, 512\}$ robots and repeat $10$ times with different random initial deployments to report their mean in performance metrics. 
We take a robot's path cost as the number of moves performed, where each takes $\tau = 1 \si{\second}$ for execution.

\subsubsection{Performance metrics}

We consider \textit{the mission time ($T_m$)} as the performance metric to compare proposed \FnConCPP with \FnOnDemCPP and \FnAPFCPP, which are horizon-based, i.e., in each horizon, the path planning and the path following interleave. 
So, in \FnOnDemCPP and \FnAPFCPP, $T_m = T_c + T_p$, where $T_c$ and $T_p$ are \textit{the total computation time} and \textit{the total path following time}, respectively, across all horizons. 
Note that $T_p = \Lambda \times \tau$, where $\Lambda$ is \textit{the path length}, which is equal for all the robots. 
However, both can overlap in an \textit{interval} (duration between consecutive $CLK$ values) in \FnConCPP. 
Hence, in \FnConCPP, $T_m$ is the duration from the beginning till complete coverage gets attained, where $T_c$ is the total computation time across all rounds and $\Lambda$ is the final $CLK$ value when the coverage gets completed.

\subsection{Results and Analysis}
\label{subsec:results_and_analysis}

We show the experimental results in Table \ref{tab:experimental_results}, where we list the workspaces in increasing order of \textit{the number of obstacle-free cells} ($|W_{free}|$) for each type of robot. 
\shortversion
{
For clarity, we do not show \textit{standard deviations} (available in the longer version of this paper~\cite{DBLP:conf/arxiv/MitraS24}). 
}

\subsubsection{Achievement in parallelizing path planning and path following}

\begin{figure}
    \centering
    \shortversion
    {
    \includegraphics[scale=0.185]{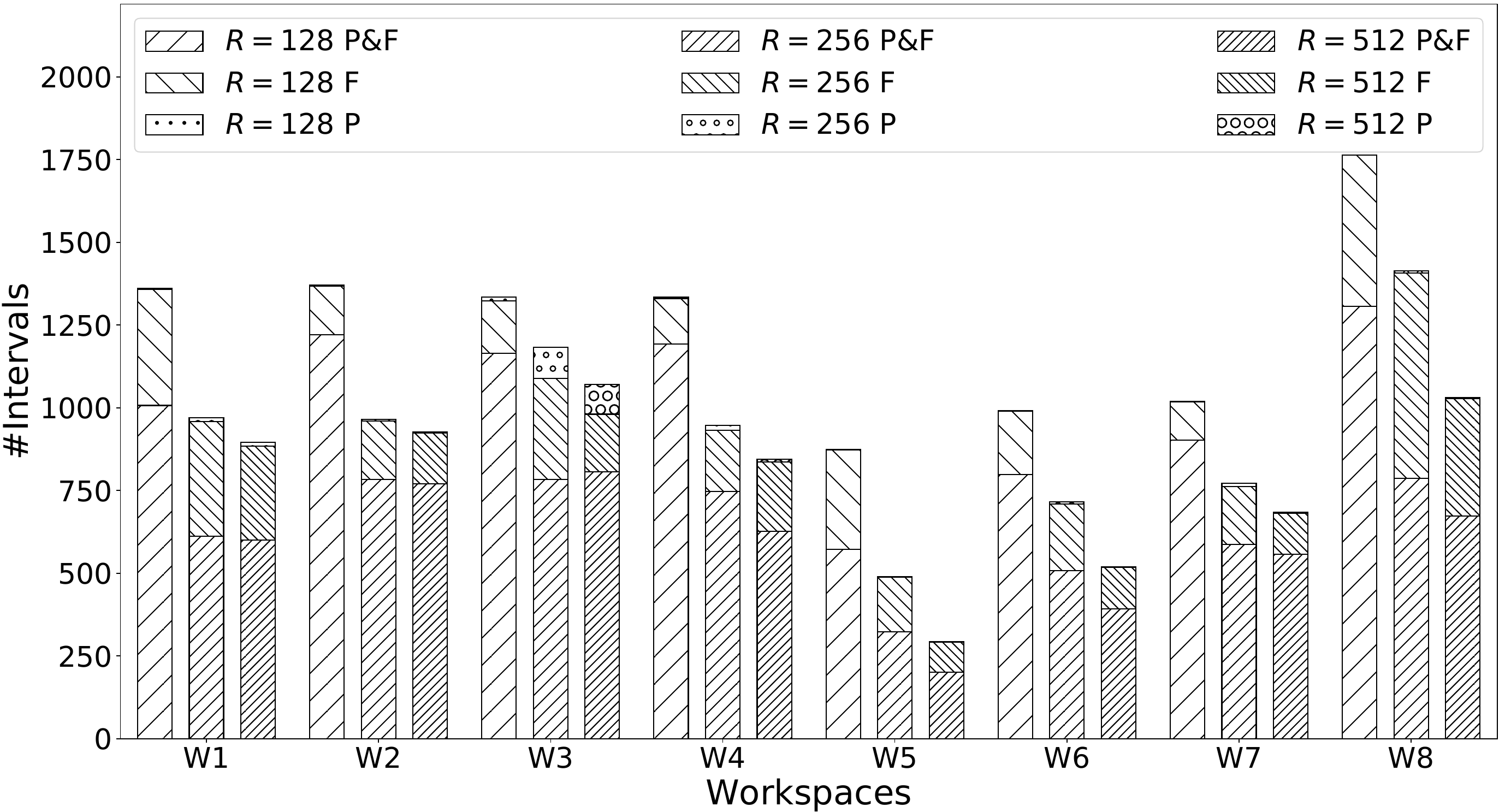}
    }
    \longversion
    {
    \includegraphics[scale=0.185]{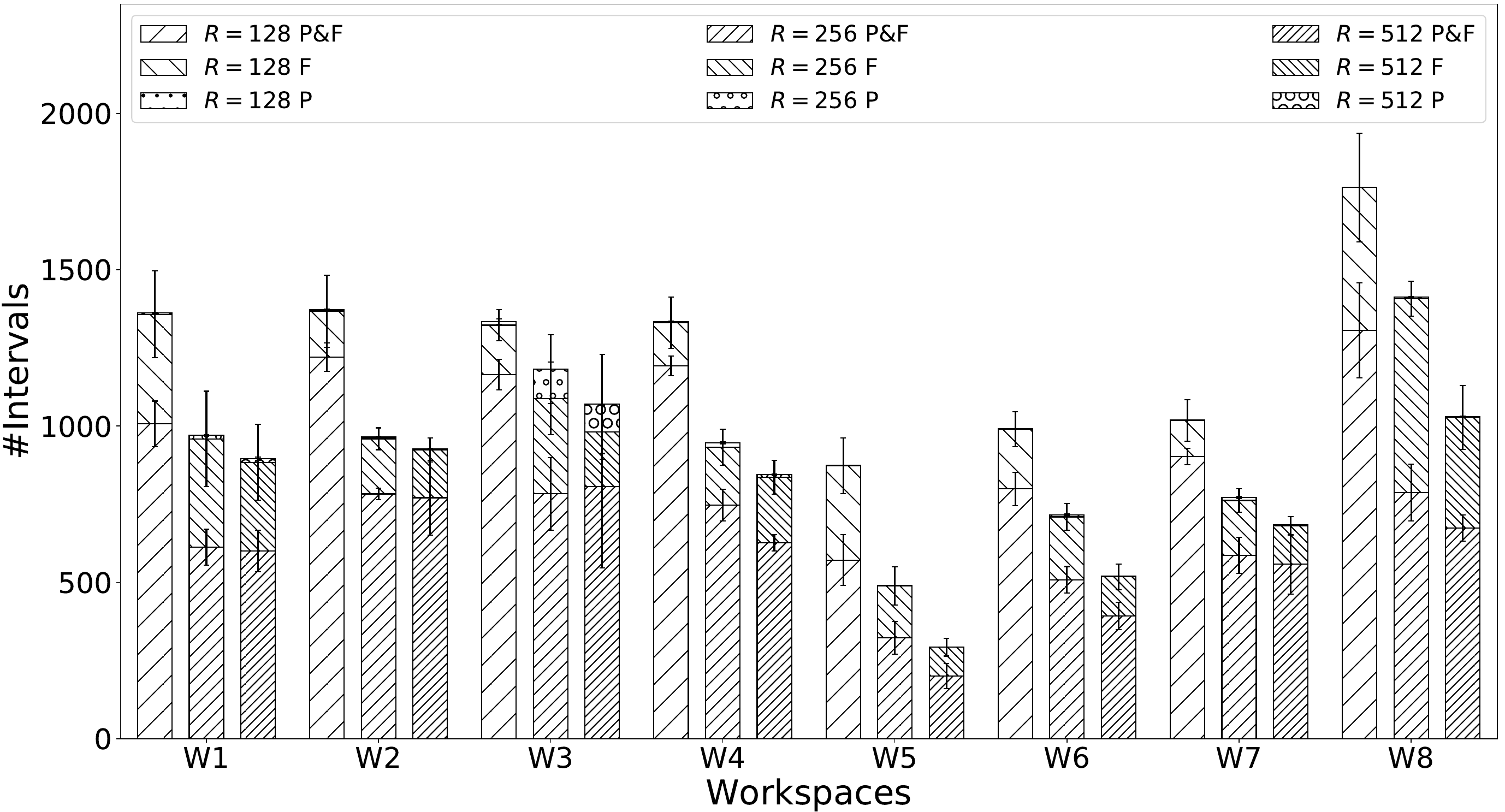}
    }
    \caption{Decomposition of \#Intervals}
    \label{fig:clk_stat}
\end{figure}

In Figure \ref{fig:clk_stat}, we show the decomposition of \textit{the number of intervals} (i.e., the final $CLK$ value) needed to attain the complete coverage into the number of intervals where \textit{both path planning and path following} happened, \textit{only path following} happened, and \textit{only path planning} happened, which we abbreviate as \textbf{P\&F}, \textbf{F}, and \textbf{P}, respectively. 
It shows the power of the proposed approach, parallelizing the path planning with the path following in $55-89\%$ of the intervals. 
Further, it shows that \#Intervals decreases as the number of robots ($R$) increases, completing the coverage faster.

\subsubsection{Performance of \texttt{ConCPPForPar} (Algorithm \ref{algo:cpp_for_par})}

The CP invokes \FnConCPPForPar in each round (line 30 in Algorithm \ref{algo:concpp}) to find the participants' timestamped paths. 
As $R$ increases, the number of participants per round $R^*$ (fourth column in Table~\ref{tab:experimental_results}) and the number of nonparticipants per round $R - R^*$ increase. 
So, Figure \ref{fig:comp_time_per_round_stat} shows that \textit{the mean computation time per round} ($\mu_c$), comprising of \textit{the mean cost-optimal path finding time} ($\mu_{COP}$) and \textit{the mean collision-free path finding time} ($\mu_{CFP}$), also increases. 
Recall that the CP spends $\mu_{COP}$ in \FnCOPForPar and $\mu_{CFP}$ in \FnCFPForPar (lines 1 and 7, respectively, in Algorithm~\ref{algo:cpp_for_par}). 
Moreover, \FnCOPForPar is independent of the timestamp $T_{start}$ but \FnCFPForPar is not as the nonparticipants remaining paths $\Sigma_{rem}$ depend on $T_{start}$ (line 6). 
\shortversion
{
As $\mu_{CFP}$ increases with $R$, \FnCFPForPar fails to find the participants' collision-free paths $\Sigma$ on time (line 9) in merely $0-8\%$ of the rounds (please see Figure 7 and read Section \Romannum{5}-B.2 for its explanation in the longer version of this paper~\cite{DBLP:conf/arxiv/MitraS24}). 
Therefore, the CP reattempts to find $\Sigma$ in the second iteration of the \textbf{while} loop (lines 3-11). 
}
\longversion
{
As $\mu_{CFP}$ increases with $R$, Figure~\ref{fig:round_stat} shows that \FnCFPForPar fails to find the participants' collision-free paths $\Sigma$ on time in merely $0-8\%$ of the rounds (line 9). 
Therefore, the CP reattempts to find $\Sigma$ in the second iteration of the \textbf{while} loop (lines 3-11). 
Note that \#Rounds decreases with $R$ as more robots complete the coverage faster. 
Also, the CP runs these rounds in \textbf{P\&F} and \textbf{P} intervals. 
Moreover, \#Rounds can be greater than \#Intervals as the CP can sequentially run multiple rounds on demand in an interval, as explained in Example~\ref{exp:concpp}. 
}

\begin{figure}
    \centering
    \shortversion
    {
    \includegraphics[scale=0.185]{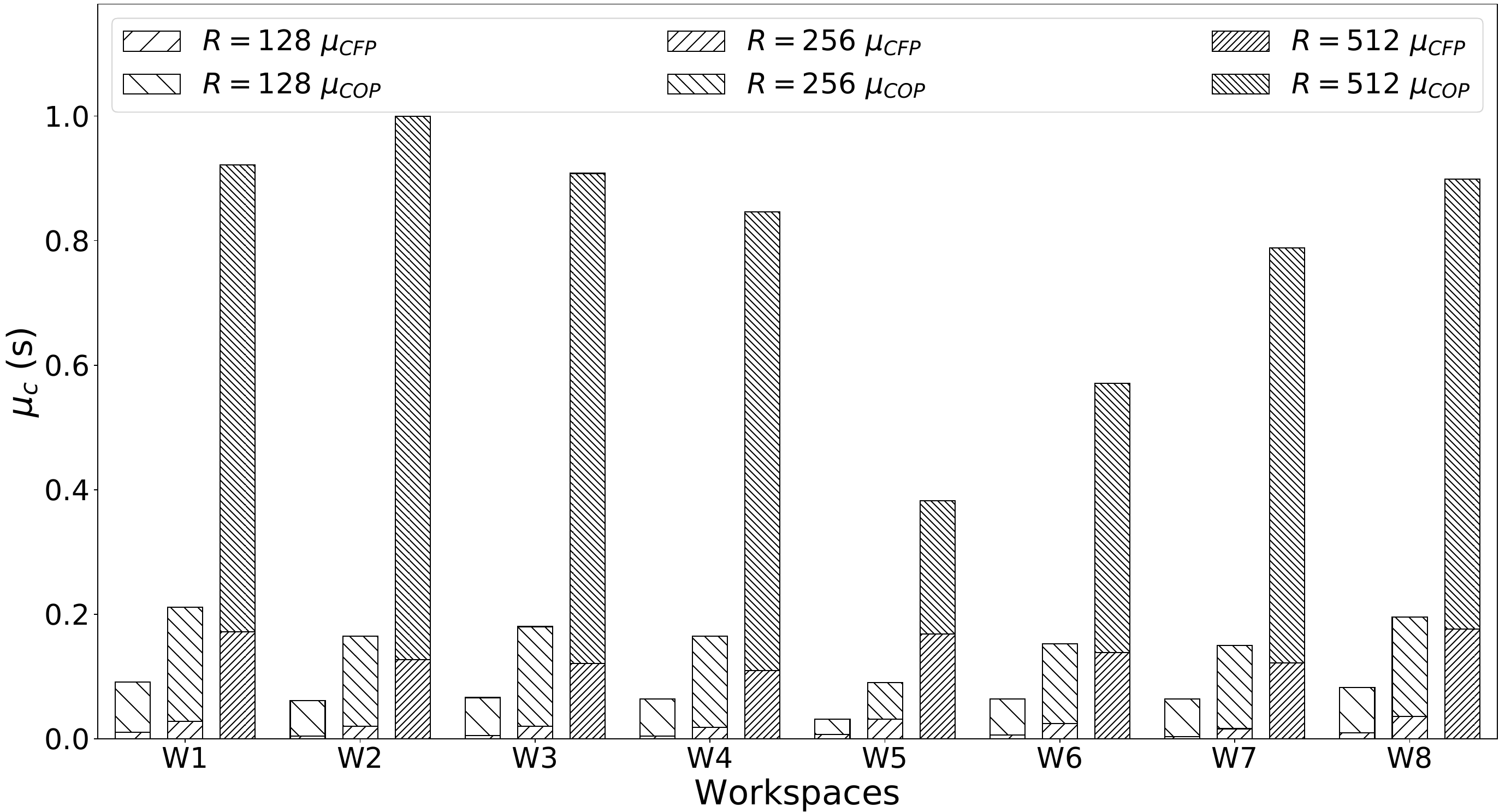}
    }
    \longversion
    {
    \includegraphics[scale=0.185]{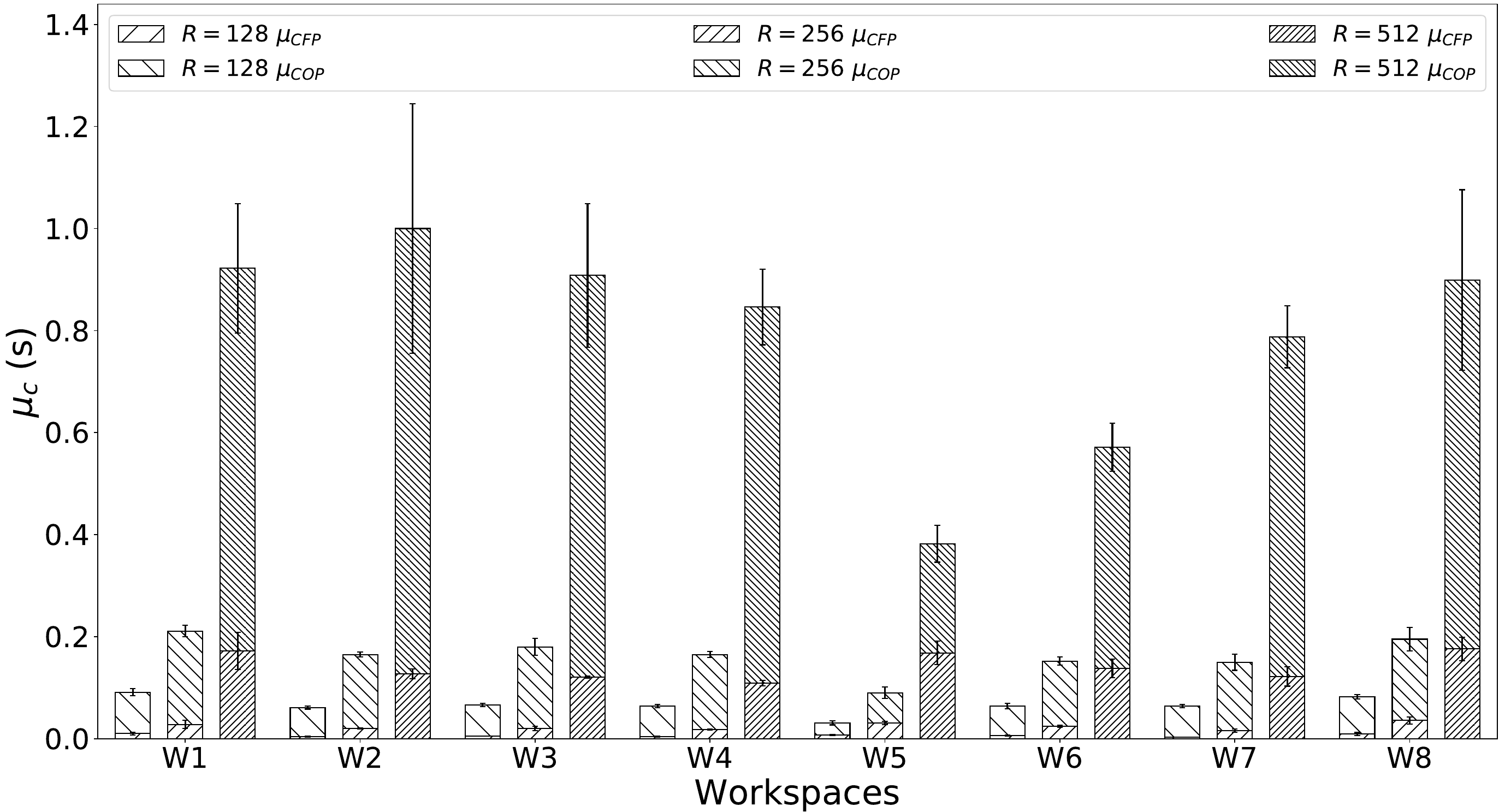}
    }
    \caption{Decomposition of $\mu_c$}
    \label{fig:comp_time_per_round_stat}
\end{figure}

\longversion
{
\begin{figure}
    \centering
    \includegraphics[scale=0.185]{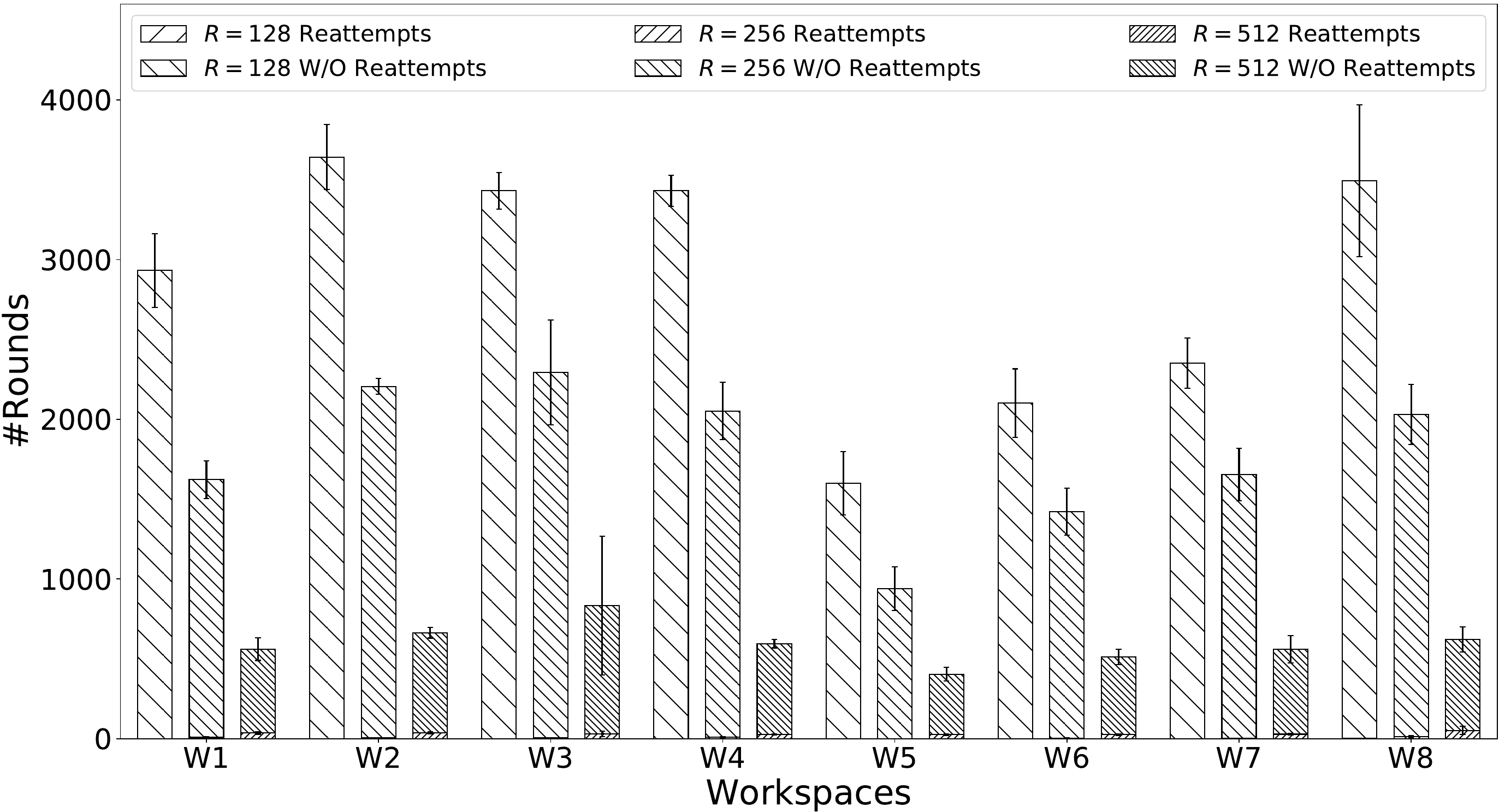}
    \caption{Decomposition of \#Rounds}
    \label{fig:round_stat}
\end{figure}

\begin{figure}
    \centering
    \includegraphics[scale=0.185]{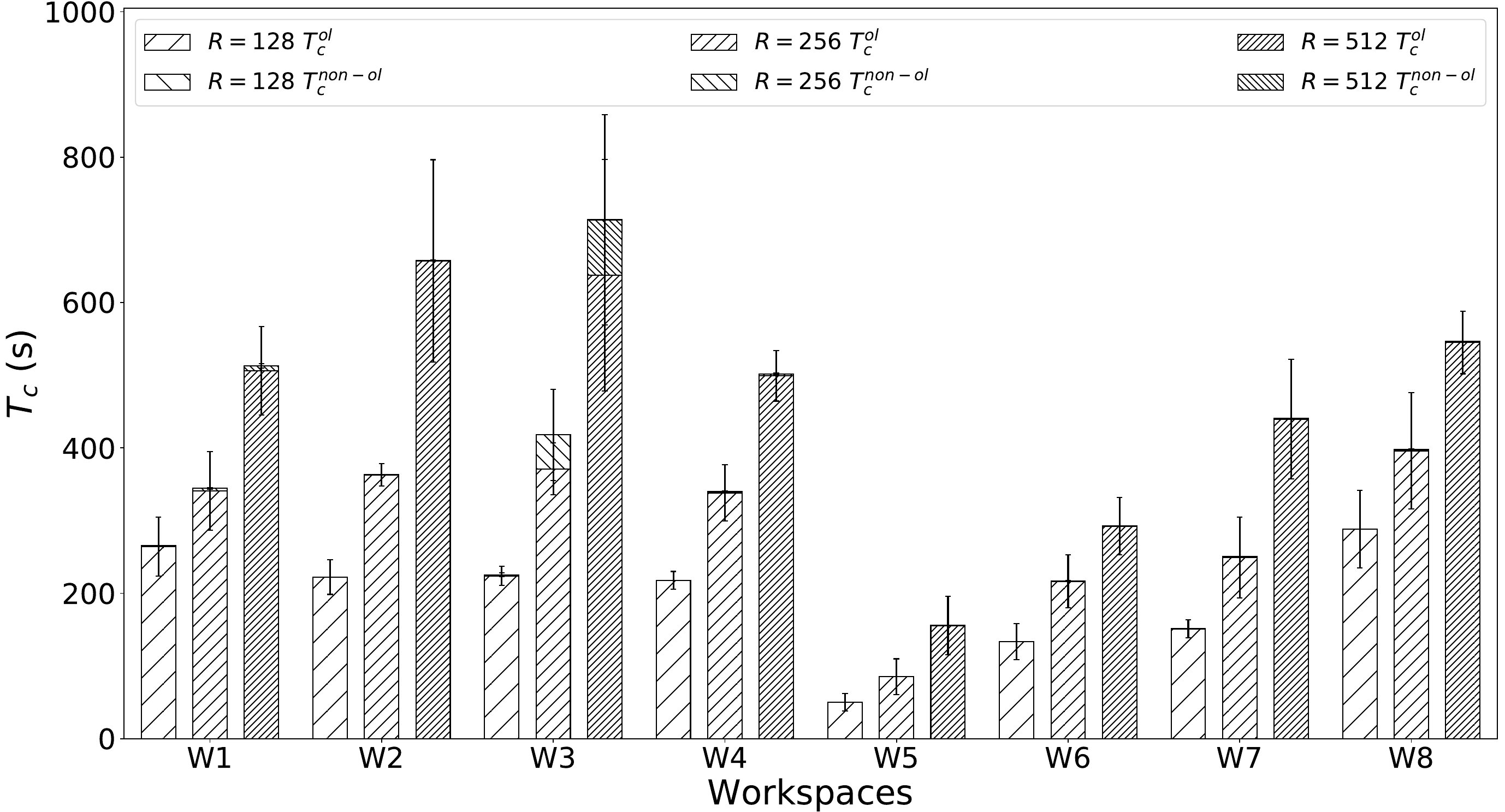}
    \caption{Decomposition of $T_c$}
    \label{fig:comp_time_stat}
\end{figure}

\begin{figure}
    \centering
    \includegraphics[scale=0.177]{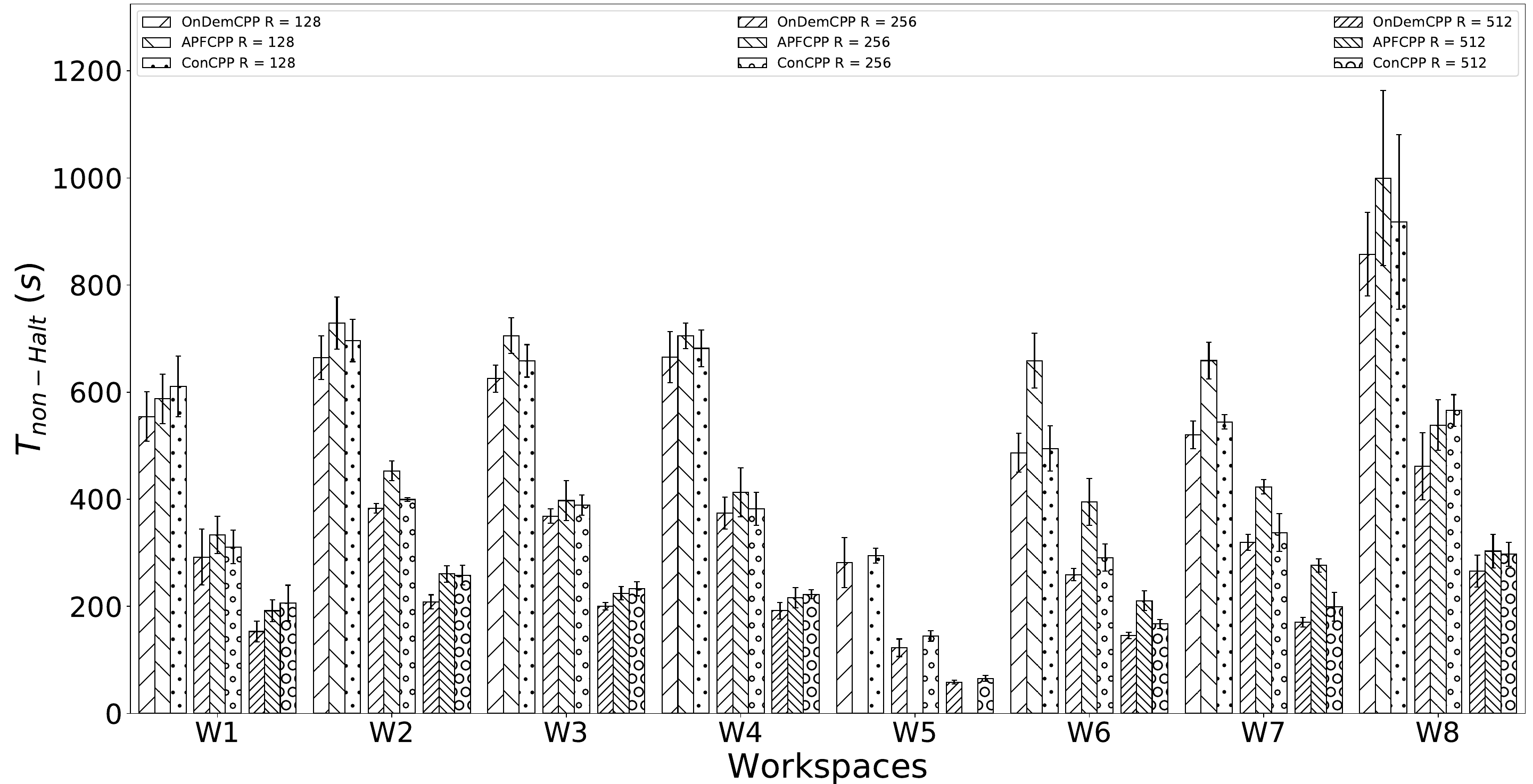}
    \caption{Time spent on non-Halt moves}
    \label{fig:non_halt}
\end{figure}
}

\subsubsection{Comparison of Mission Time in Table \ref{tab:experimental_results}}

\shortversion
{

\input{tables/experimental_results}
}
\longversion
{

\input{tables/experimental_results_long}
}
First, \textit{the number of participants per planning round} $R^*$ increases with $R$ as more robots become participants. 
In \FnOnDemCPP, inactive participants of the current horizon participate again in the next horizon with the new participants, who reach their goals in the current horizon. 
However, in \FnConCPP, the next planning round starts instantly even if $I_{par}$ contains only one participant, which could be an inactive participant of the current round or a new participant (lines 36-39 or 40-42 in Algorithm \ref{algo:concpp}). 
So, $R^*$ is lesser in \FnConCPP. 
Recall that in \FnAPFCPP, all the $R$ robots participate in each horizon. 

Next, \textit{the total computation time} $T_c$ increases with $R$ because in each round, finding cost-optimal paths and subsequently collision-free paths for $R^*$ participants while respecting $R - R^*$ non-participants' paths become intensive as Figure \ref{fig:comp_time_per_round_stat} shows. 
As $R^*$ is lesser in \FnConCPP, $T_c$ is also lesser in \FnConCPP. 
\shortversion{Out}\longversion{As Figure. \ref{fig:comp_time_stat} shows, out} of $T_c$, the total computation time that overlaps with the path following in \textbf{P\&F} intervals is denoted by $T^{ol}_c$, which is $81-99\%$ of $T_c$. 
The results indicate that almost all the computations have happened in parallel, with at least one robot following its path. 
It firmly establishes the strength of the concurrent CPP framework. 
Despite \FnAPFCPP's low time complexity, its $T_c$ exceeds that of \FnConCPP's as $R$ increases because all $R$ robots participate in  \FnAPFCPP. 
Also, in large complex workspaces, its dead-end recovery procedure gets triggered for many robots in most horizons. 
It gets even worse with the relatively complex motions of TurtleBots. 

Now, \textit{the total path following time} $T_p$ decreases with $R$ as more robots complete the coverage faster. 
In a horizon of \FnOnDemCPP, active robots that do not reach their goals but only progress also send their updated local views to the CP, making the global view more information-rich. 
So, in the next horizon, the CP finds a superior goal assignment for the participants. 
In contrast, in \FnConCPP, active robots send their updated local views to the CP after reaching their goals, keeping the global view less information-rich. 
It results in inferior goal assignment, leading the participants to distant unassigned goals, making $\Lambda$ relatively longer and so $T_p$. 
On the other hand, goal assignment is even superior in \FnAPFCPP compared to \FnOnDemCPP as all robots participate in each horizon, making $\Lambda$ relatively shorter. 
But, dead-end recovery of a robot blindly depends on attractive forces (based on the Euclidean distance only) from goals, failing to incorporate the presence of obstacles and other robots on the paths for feasibility. 
As a result, either actual paths are longer than they appear or robots fail to cover highly complex workspaces with very narrow passageways like a maze in a reasonable time limit of $45$\si{\minute} due to infeasible paths, which exposes the vulnerability of any potential field-based algorithms. 

\FnConCPP gains in terms of $T_c$ but losses in terms of $T_p$. 
Despite that, efficient overlapping of path plannings with executions makes \FnConCPP take shorter \textit{mission time} $T_m$ than its counterparts, thereby achieving a \textit{speedup} of up to $1.6\times$.

\subsubsection{Significance in Energy Consumption}

A quadcopter remains airborne during the entire $T_m$ and keeps consuming energy. 
As \FnConCPP completes the mission up to $1.6\times$ faster than its counterparts, the quadcopters spend significantly less energy during the mission. 
Out of $T_p$, the TurtleBots spent $T_{non-Halt}$ time (\shortversion{eighth column in Table~\ref{tab:experimental_results}}\longversion{in Figure \ref{fig:non_halt}}) in executing non-$\mathtt{Halt}$ moves while the rest are spent on $\mathtt{Halt}$ moves, which get prefixed to the paths during collision avoidance. 
In \FnOnDemCPP and \FnAPFCPP, a TurtleBot stays \textit{idle} during $T_{Halt} + T_c$ and saves energy. 
$T_{non-Halt}$ is larger in \FnConCPP than in \FnOnDemCPP by $1-22\%$. 
So, the TurtleBots spent more energy in \FnConCPP than in \FnOnDemCPP. 
But, as discussed above, attractive forces often mislead the assignment of a goal to a robot, which needs to avoid unaccounted obstacles and other robots en route, making $T_{non-Halt}$ more in \FnAPFCPP compared to \FnConCPP in most of the cases. 
Generally, the ground robots' energy consumption is less severe than the aerial robots.

\subsection{Simulations and Real Experiments}

For validation, we perform Gazebo~\cite{key_gazebo} 
simulations on six $2$D benchmark workspaces from~\cite{stern2019mapf} with $10$ IRIS quadcopters and $10$ TurtleBot3 robots, respectively. 
We also demonstrate two real experiments (see Figure \ref{fig:real_exp}) - one outdoor with three quadcopters and one indoor with two TurtleBot2 robots. 
We equip each quadcopter with a Cube Orange for autopilot, a Here$3$ GPS for localization, and a HereLink Air Unit for communication with the Remote Controller. 
We equip each TurtleBot2 robot with $4$ HC SR$04$ ultrasonic sound sensor for obstacle detection in the Vicon~\cite{key_vicon}-enabled indoor workspace for localization.  
These experimental videos are available 
at \textcolor{blue}{https://youtu.be/M1FGrU6hty8}.

\begin{figure}[t]
    \centering
        \begin{subfigure}{0.35\linewidth}
        \centering
        \includegraphics[scale=0.35]{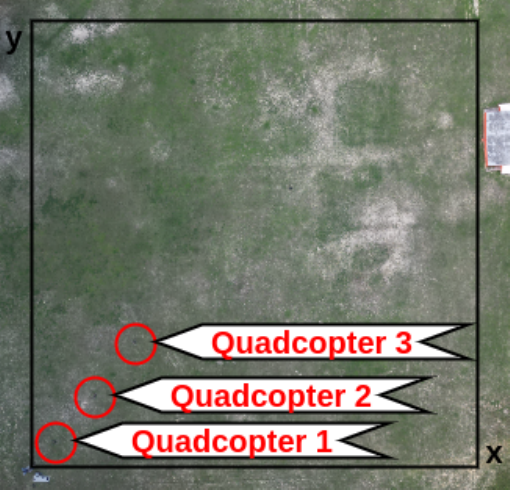}
        \caption{$10\times10$ \ \ \ \ \ \ \ \ outdoor workspace \ \ \ \ \ \ \ \ \ \ \ \ \ \ \ \ \ \ \ \ \ \ \ \ \ \ \ \ \ \ \ \ [cell size = $5\si{\meter}$]}       
        \label{subfig:outdoor_ws}
    \end{subfigure}
    \hfill
    \begin{subfigure}{0.59\linewidth}
        \centering
        \includegraphics[scale=0.038]{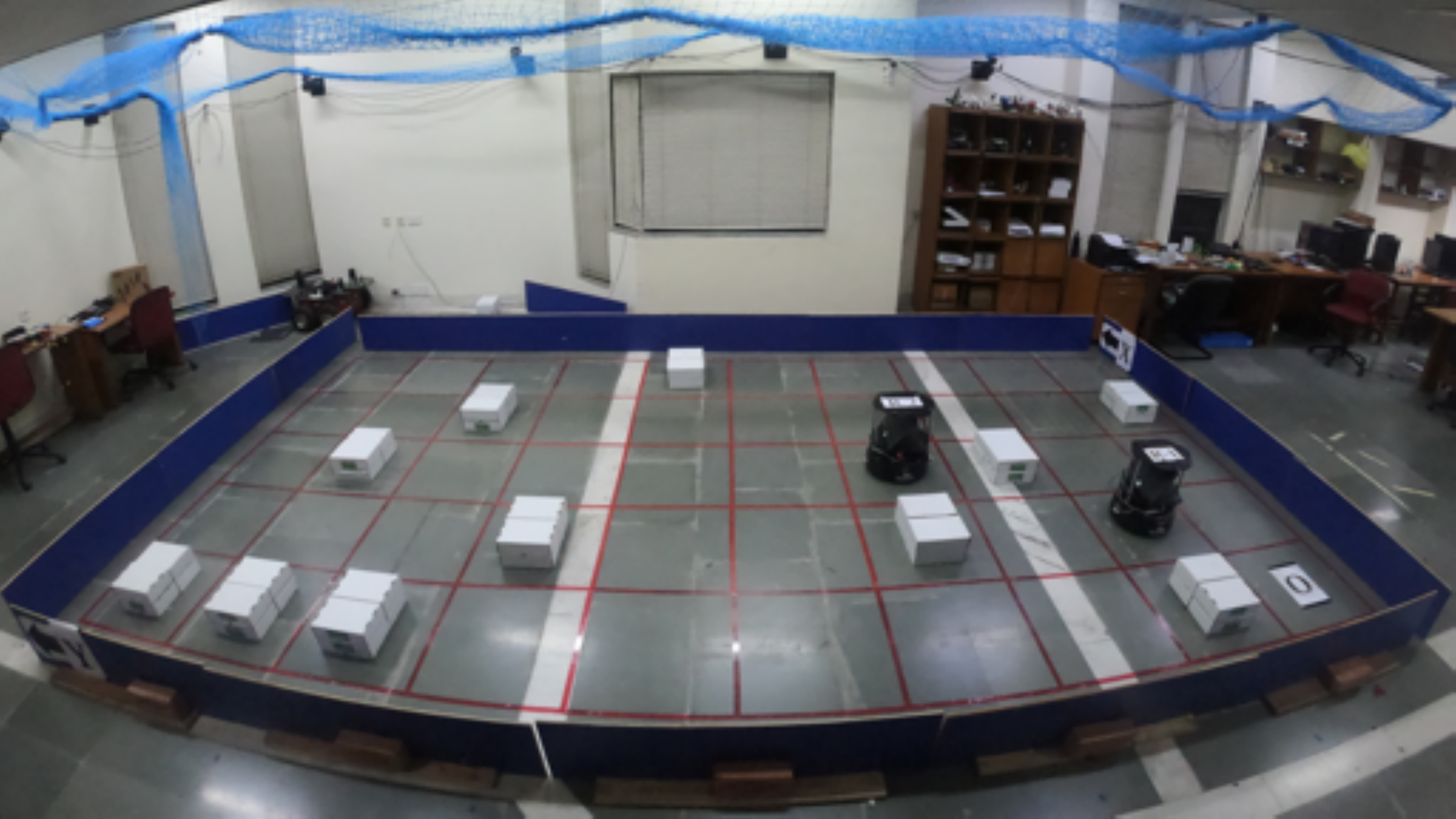}
        \caption{$\phantom{0}5\times10$ indoor workspace \ \ \ \ \ \ \ \ \ \ \ \ \ \ \ \ [cell size = $0.61\si{\meter}$] \\}       
        \label{subfig:indoor_ws}
    \end{subfigure}
    \caption{Workspaces for the real experiments, where the robot size determines the grid cell size}
    \label{fig:real_exp}
\end{figure}

%% file: tables/experimental_results.tex
\begin{table*}[t]
    \caption{Experimental results}
    \label{tab:experimental_results}
    \centering
    \resizebox{0.99\textwidth}{!}      
    {
        \begin{tabular}{@{}lcccrrrrrrrrrrrrrrrcc@{}}
            \toprule

             & & & & \multicolumn{2}{c}{$R^*$} & \multicolumn{3}{c}{$T_c\ (\si{\second})$} & $T^{ol}_c\ (\si{\second})$ & \multicolumn{3}{c}{$T_p\ (\si{\second})$} & \multicolumn{3}{c}{$T_{non-Halt}\ (\si{\second})$} & \multicolumn{3}{c}{$T_m\ (\si{\second})$} & \multicolumn{2}{c}{\textbf{Speed Up}} \\
             
            \cmidrule(lr){5-6}
            \cmidrule(lr){7-9}
            \cmidrule(lr){10-10}
            \cmidrule(lr){11-13}
            \cmidrule(lr){14-16}
            \cmidrule(lr){17-19}
            \cmidrule(lr){20-21}
            
            $M$ & \multicolumn{2}{c}{\textit{Workspace}} & $R$ & \FnOnDem & \FnCon & \FnOnDem & \FnAPF & \FnCon & \FnCon & \FnOnDem & \FnAPF & \FnCon & \FnOnDem & \FnAPF & \FnCon & \FnOnDem & \FnAPF & \FnCon & \FnOnDem & \FnAPF \\
            
            \midrule
            
            \multirow{12}{*}{\begin{turn}{90}Quadcopter\end{turn}} & \multirow{3}{*}{\shortstack{W1:\\w\_woundedcoast\\  $578 \times 642\ (34,020)$}} & \multirow{3}{*}{\includegraphics[scale=0.025]{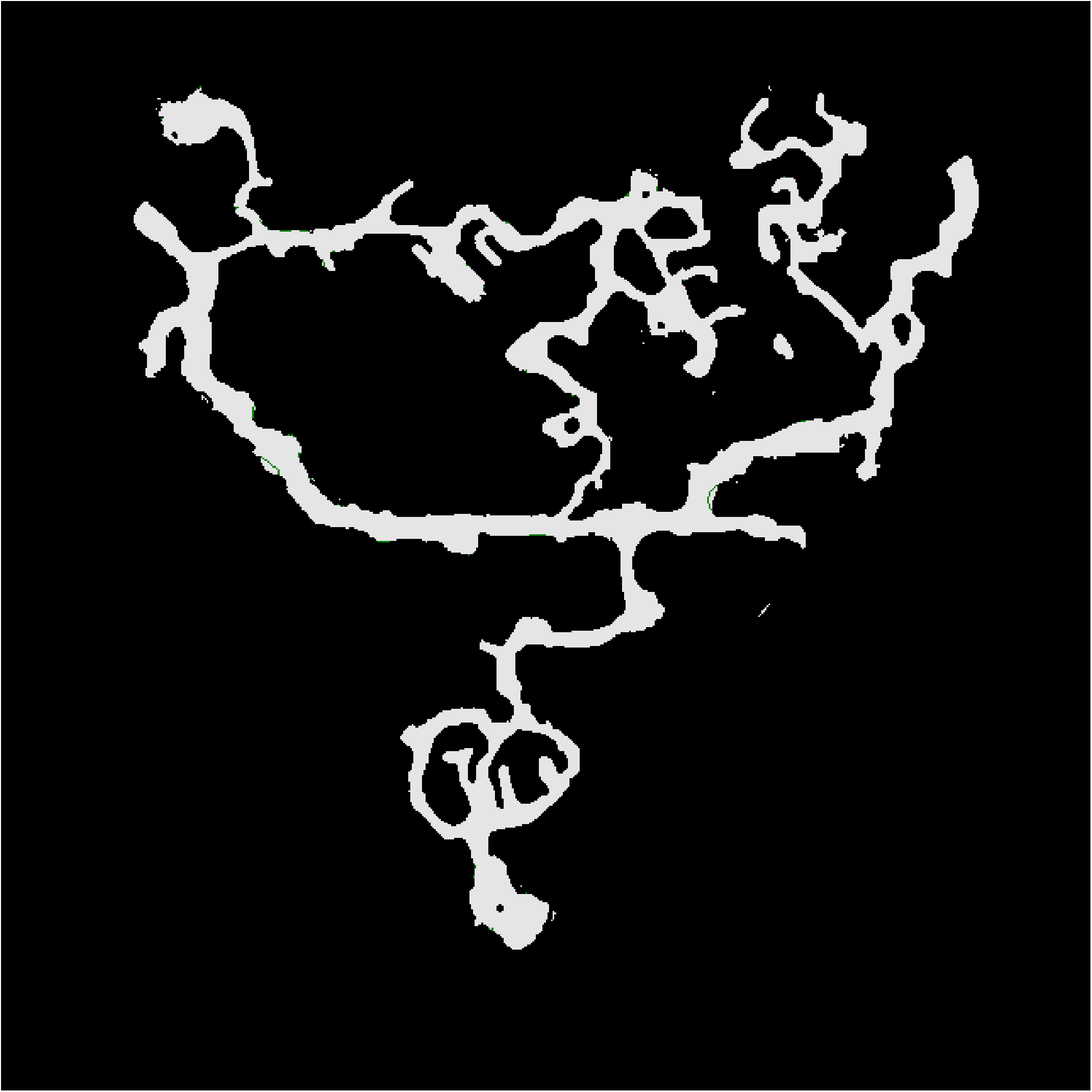}} & 128 & \phantom{0}80.4 & 33.8 & 410.1 & 193.2 & 265.3 & 264.4 & 1042.0 & 895.6 & 1358.0 & 554.5 & 587.5 & 610.5 & 1452.1 & 1146.5 & 1358.2 & \textbf{1.1} & 0.8 \\
             &  &  & 256 & 170.7 & 82.6 & 579.2 & 417.7 & 344.2 & 340.9 & 804.2 & 588.8 & 962.0 & 291.9 & 333.5 & 310.9 & 1383.4 & 1083.6 & 962.1 & \textbf{1.4} & \textbf{1.1} \\
             &  &  & 512 & 361.5 & 233.9 & 639.5 & 558.4 & 512.8 & 506.1 & 557.2 & 395.6 & 891.0 & 153.0 & 191.8 & 206.2 & 1196.7 & 1048.2 & 891.0 & \textbf{1.3} & \textbf{1.2} \\
             
            \cmidrule(lr){2-21}
            
             & \multirow{3}{*}{\shortstack{W2:\\Paris\_1\_256\\ $256 \times 256\ (47,240)$}} & \multirow{3}{*}{\includegraphics[scale=0.025]{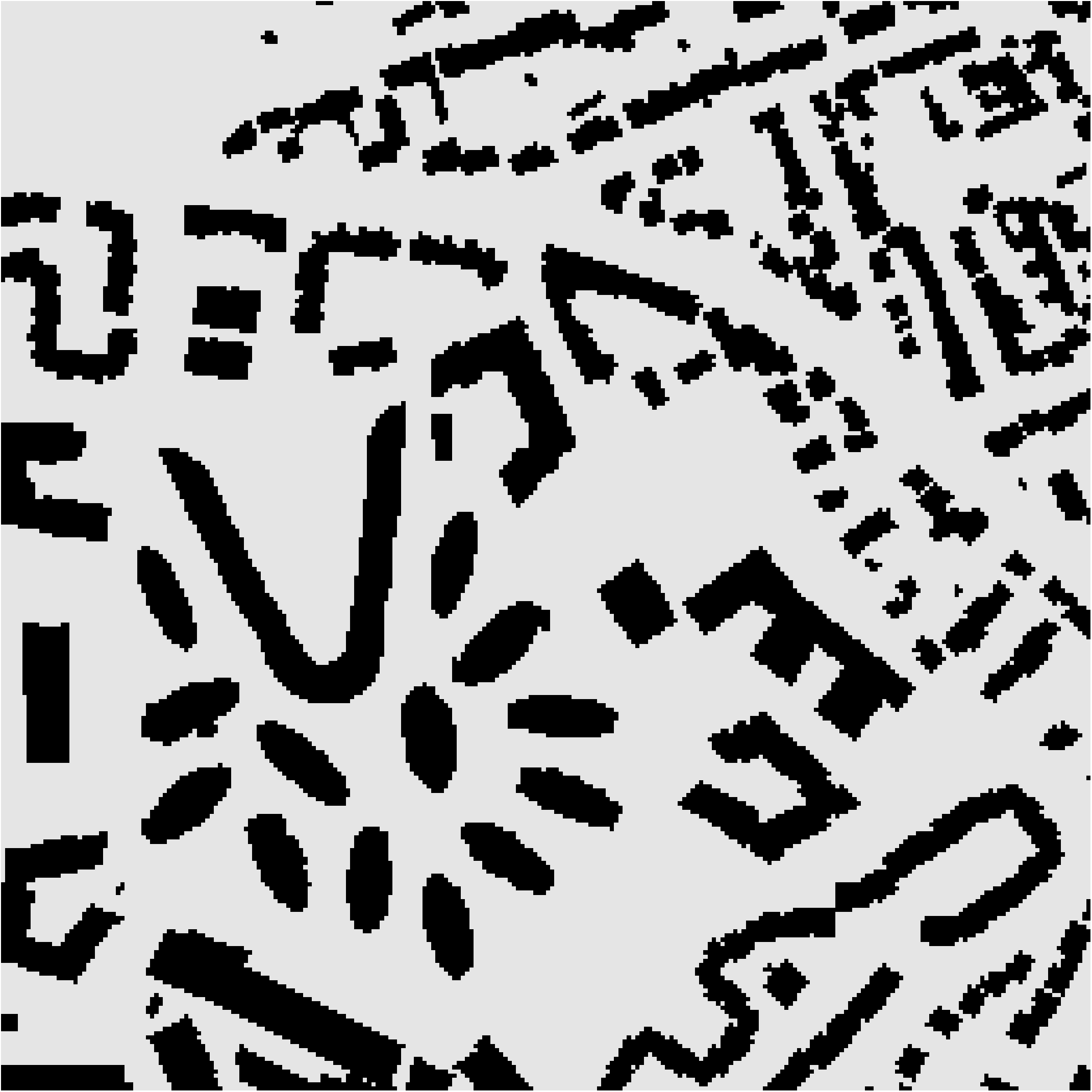}} & 128 & \phantom{0}81.3 & 30.3 & 332.5 & 158.3 & 222.3 & 222.2 & 962.6 & 853.0 & 1371.6 & 664.3 & 729.0 & 696.5 & 1295.1 & 1082.0 & 1371.7 & 0.9 & 0.8 \\
             &  &  & 256 & 157.2 & 62.6 & 618.5 & 338.7 & 363.3 & 362.9 & 656.2 & 573.4 & 965.8 & 383.1 & 453.1 & 399.5 & 1274.7 & 1014.9 & 965.9 & \textbf{1.3} & \textbf{1.1} \\
             &  &  & 512 & 335.6 & 194.4 & 992.2 & 704.6 & 658.0 & 657.2 & 455.7 & 373.3 & 929.0 & 208.5 & 260.5 & 258.3 & 1447.9 & 1195.6 & 929.1 & \textbf{1.6} & \textbf{1.3} \\
             
            \cmidrule(lr){2-21}
            
             & \multirow{3}{*}{\shortstack{W3:\\Berlin\_1\_256\\ $256 \times 256\ (47,540)$}} & \multirow{3}{*}{\includegraphics[scale=0.025]{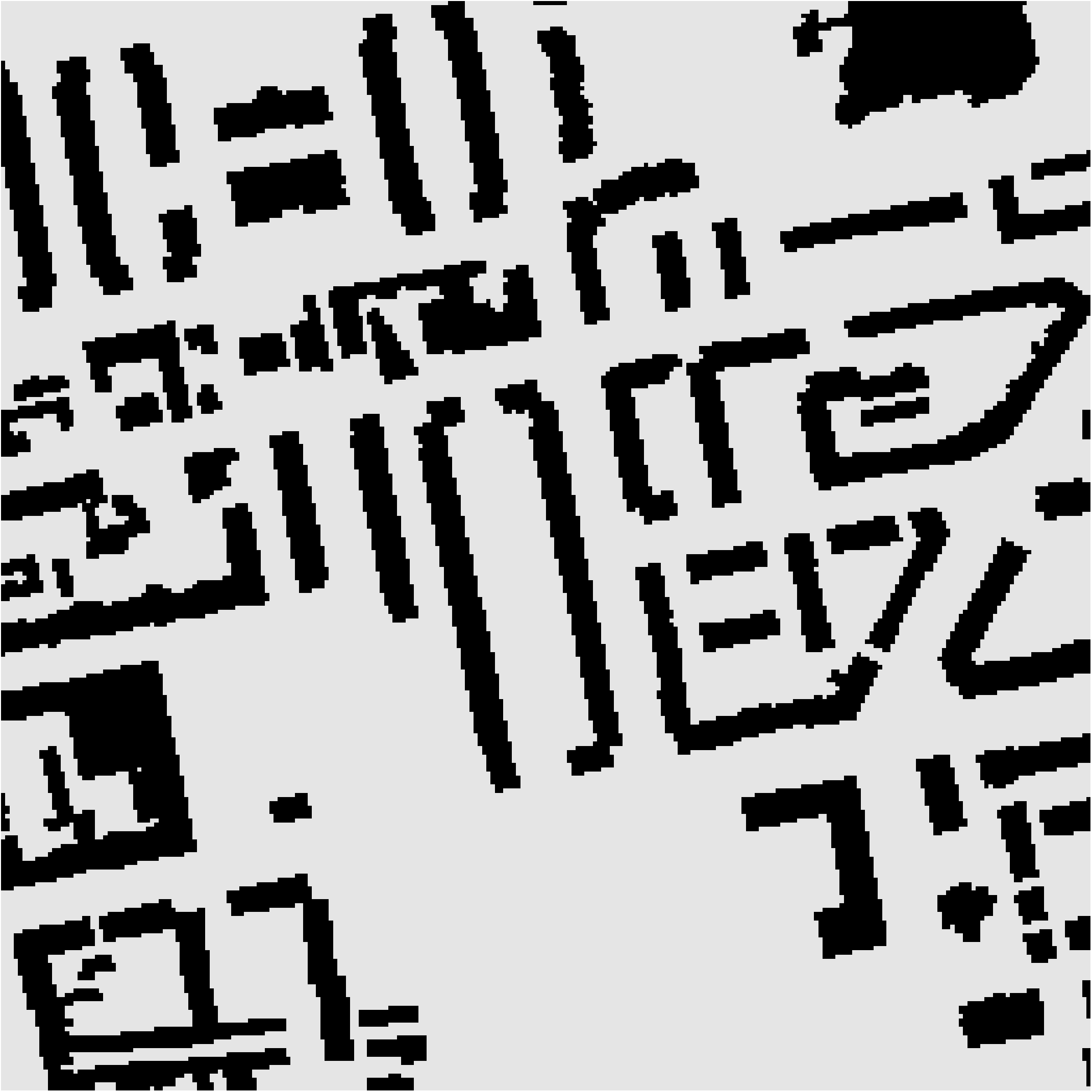}} & 128 & \phantom{0}83.2 & 31.9 & 296.2 & 170.6 & 225.4 & 224.0 & 898.2 & 739.0 & 1334.0 & 625.2 & 705.5 & 658.4 & 1194.4 & 978.1 & 1334.1 & 0.9 & 0.7 \\
             &  &  & 256 & 167.4 & 80.8 & 598.9 & 393.5 & 417.9 & 371.1 & 669.6 & 533.2 & 1186.8 & 368.5 & 397.5 & 389.1 & 1268.5 & 1018.1 & 1186.9 & \textbf{1.1} & 0.9 \\
             &  &  & 512 & 365.2 & 259.8 & 947.8 & 911.6 & 713.8 & 637.5 & 524.9 & 353.8 & 1075.6 & 200.0 & 224.3 & 232.7 & 1472.7 & 1371.9 & 1075.7 & \textbf{1.4} & \textbf{1.3} \\
             
            \cmidrule(lr){2-21}
            
             & \multirow{3}{*}{\shortstack{W4:\\Boston\_0\_256\\ $256 \times 256\ (47,768)$}} & \multirow{3}{*}{\includegraphics[scale=0.025]{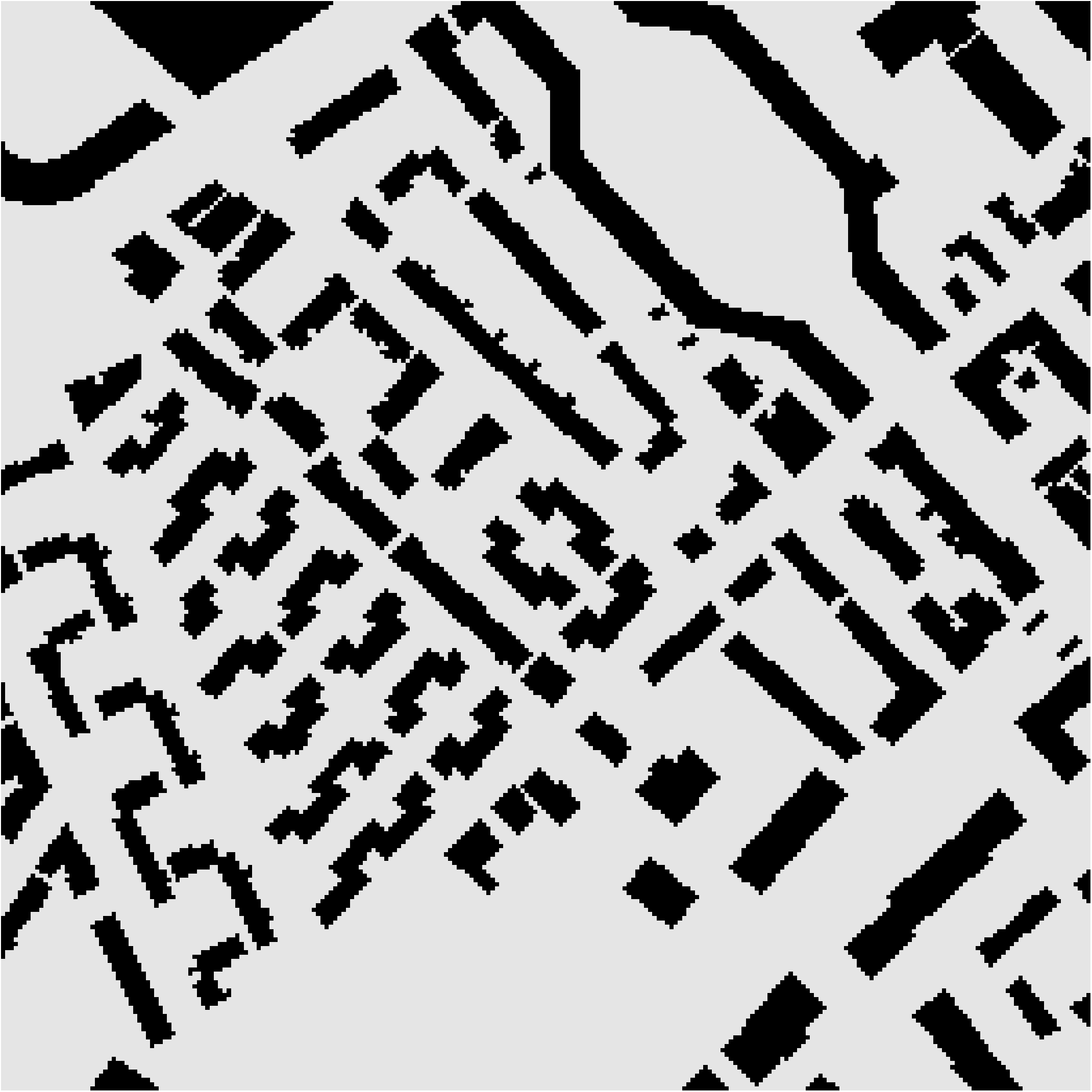}} & 128 & \phantom{0}79.6 & 30.8 & 313.3 & 259.0 & 217.8 & 217.6 & 928.0 & 730.0 & 1334.8 & 665.3 & 705.1 & 681.8 & 1241.3 & 1057.4 & 1334.9 & 0.9 & 0.8 \\
             &  &  & 256 & 157.8 & 66.9 & 531.7 & 494.6 & 339.8 & 338.2 & 625.9 & 428.6 & 949.8 & 373.8 & 412.9 & 382.1 & 1157.6 & 1017.4 & 949.9 & \textbf{1.2} & \textbf{1.1} \\
             &  &  & 512 & 345.3 & 187.0 & 718.1 & 766.5 & 501.8 & 499.0 & 428.7 & 225.2 & 847.6 & 191.9 & 216.5 & 222.4 & 1146.8 & 1097.9 & 847.7 & \textbf{1.4} & \textbf{1.3} \\
             
            \cmidrule(lr){1-21}
            
            \multirow{12}{*}{\begin{turn}{90}TurtleBot\end{turn}} & \multirow{3}{*}{\shortstack{W5:\\maze-128-128-2\\ $128 \times 128\ (10,858)$}} & \multirow{3}{*}{\includegraphics[scale=0.025]{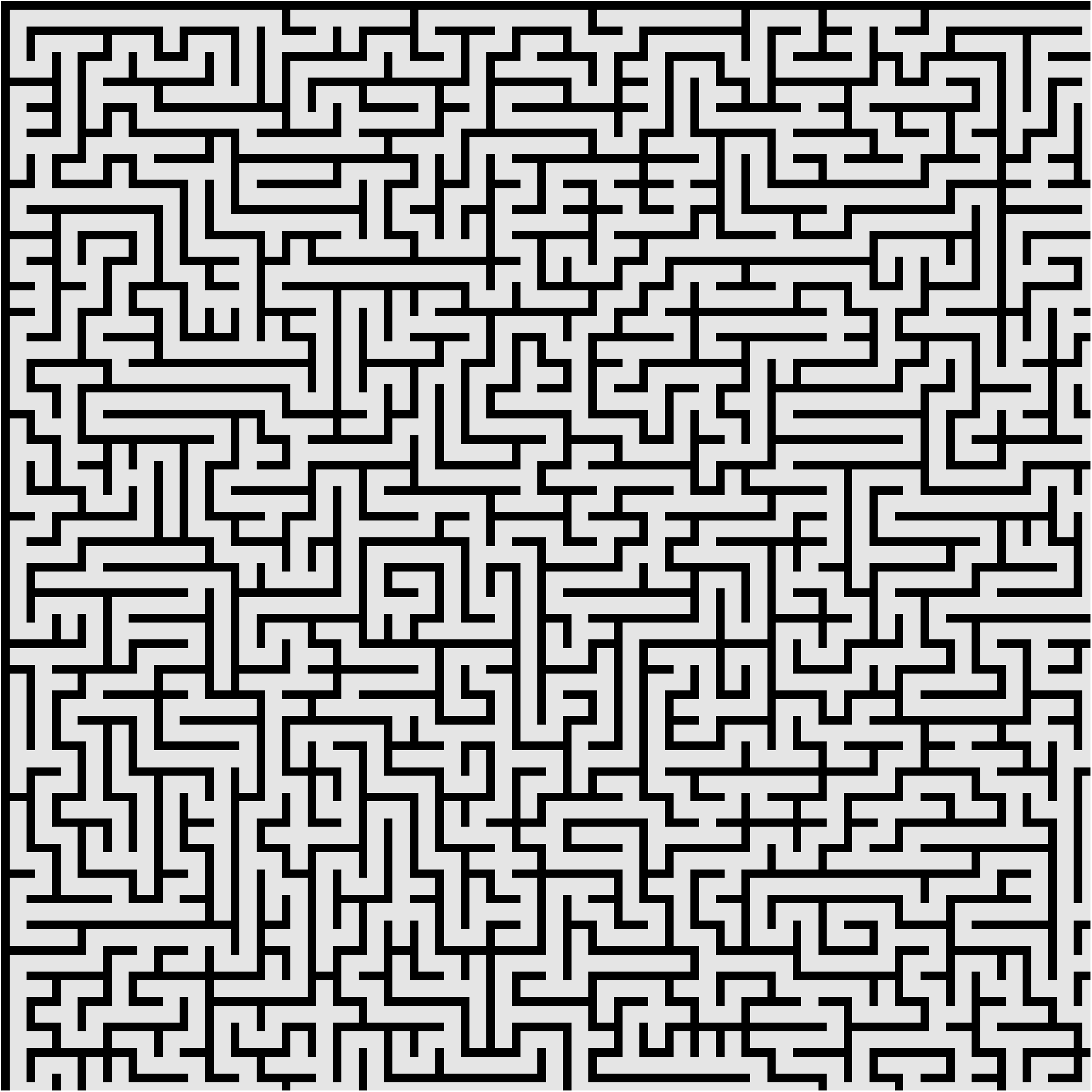}} & 128 & \phantom{0}75.5 & 47.8 & 76.8 &\textcolor{red}{TO}  & 50.1 & 50.0 & 791.5 &\textcolor{red}{TO}  & 874.6 & 281.6 &\textcolor{red}{TO}  & 294.8 & 868.3 &\textcolor{red}{TO}  & 874.7 & 0.9 &\textcolor{red}{TO}  \\
             &  &  & 256 & 174.7 & 111.6 & 97.5 &\textcolor{red}{TO}  & 85.4 & 85.3 & 417.6 &\textcolor{red}{TO}  & 489.8 & 122.7 &\textcolor{red}{TO}  & 144.8 & 515.1 &\textcolor{red}{TO}  & 489.9 & \textbf{1.1} &\textcolor{red}{TO}  \\
             &  &  & 512 & 378.4 & 301.8 & 155.5 &\textcolor{red}{TO}  & 156.1 & 155.5 & 224.3 &\textcolor{red}{TO}  & 293.4 & 58.3 &\textcolor{red}{TO}  & 65.2 & 379.8 &\textcolor{red}{TO}  & 293.5 & \textbf{1.3} &\textcolor{red}{TO}  \\
             
            \cmidrule(lr){2-21}
            
             & \multirow{3}{*}{\shortstack{W6:\\den520d\\ $257 \times 256\ (28,178)$}} & \multirow{3}{*}{\includegraphics[scale=0.025]{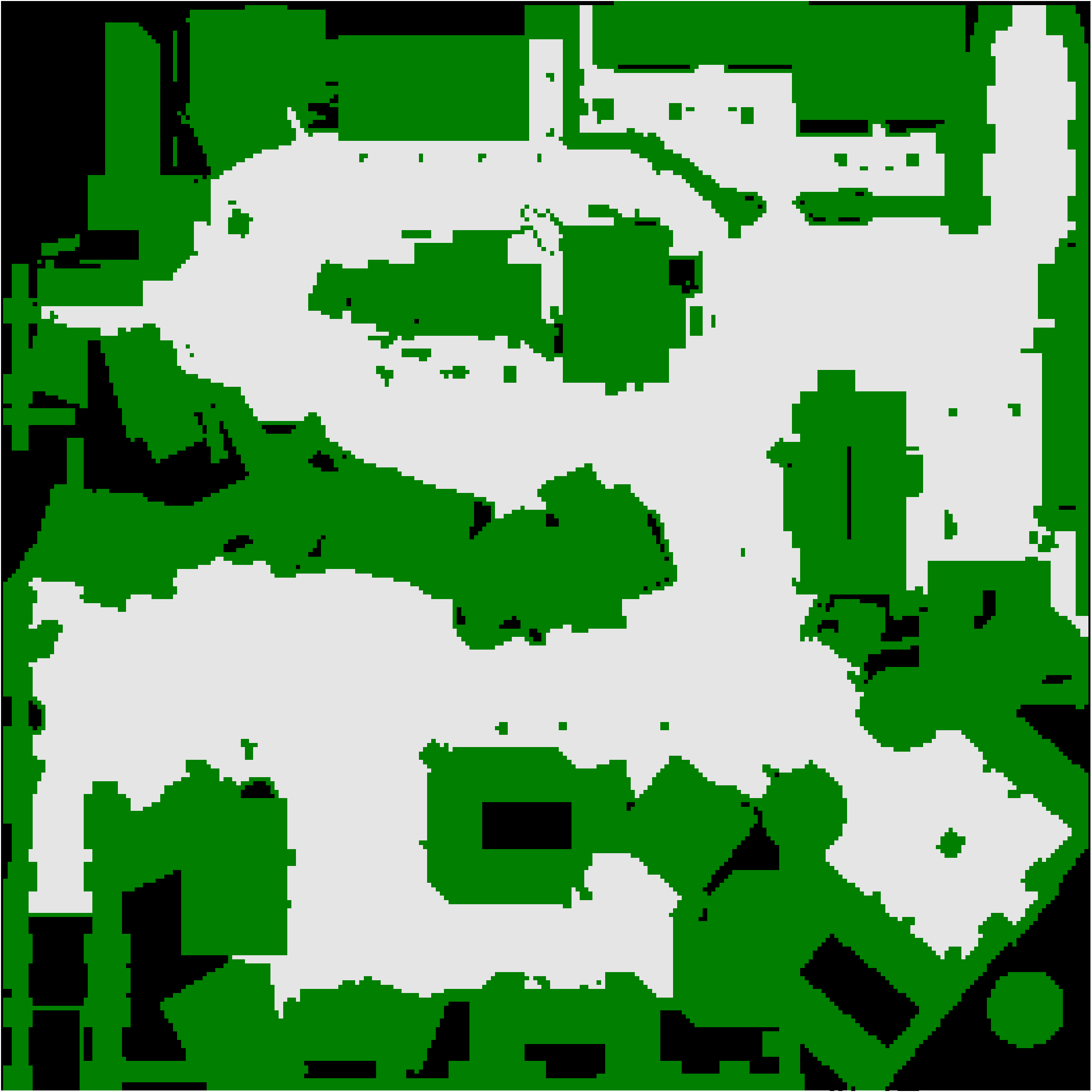}} & 128 & \phantom{0}71.9 & 31.4 & 191.5 & 138.2 & 133.6 & 133.6 & 773.1 & 731.3 & 991.0 & 486.9 & 658.6 & 494.6 & 964.6 & 953.9 & 991.1 & 0.9 & 0.9 \\
             &  &  & 256 & 153.8 & 68.0 & 307.6 & 316.9 & 217.0 & 216.4 & 506.4 & 463.9 & 715.8 & 259.3 & 395.4 & 291.0 & 814.0 & 873.8 & 715.9 & \textbf{1.1} & \textbf{1.2} \\
             &  &  & 512 & 331.1 & 200.4 & 384.9 & 371.0 & 292.7 & 292.2 & 398.6 & 343.0 & 519.4 & 145.6 & 210.4 & 167.0 & 783.5 & 814.9 & 519.5 & \textbf{1.5} & \textbf{1.6} \\
             
            \cmidrule(lr){2-21}
            
             & \multirow{3}{*}{\shortstack{W7:\\warehouse-20-40-10-2-2\\ $164 \times 340\ (38,756)$}} & \multirow{3}{*}{\includegraphics[scale=0.025]{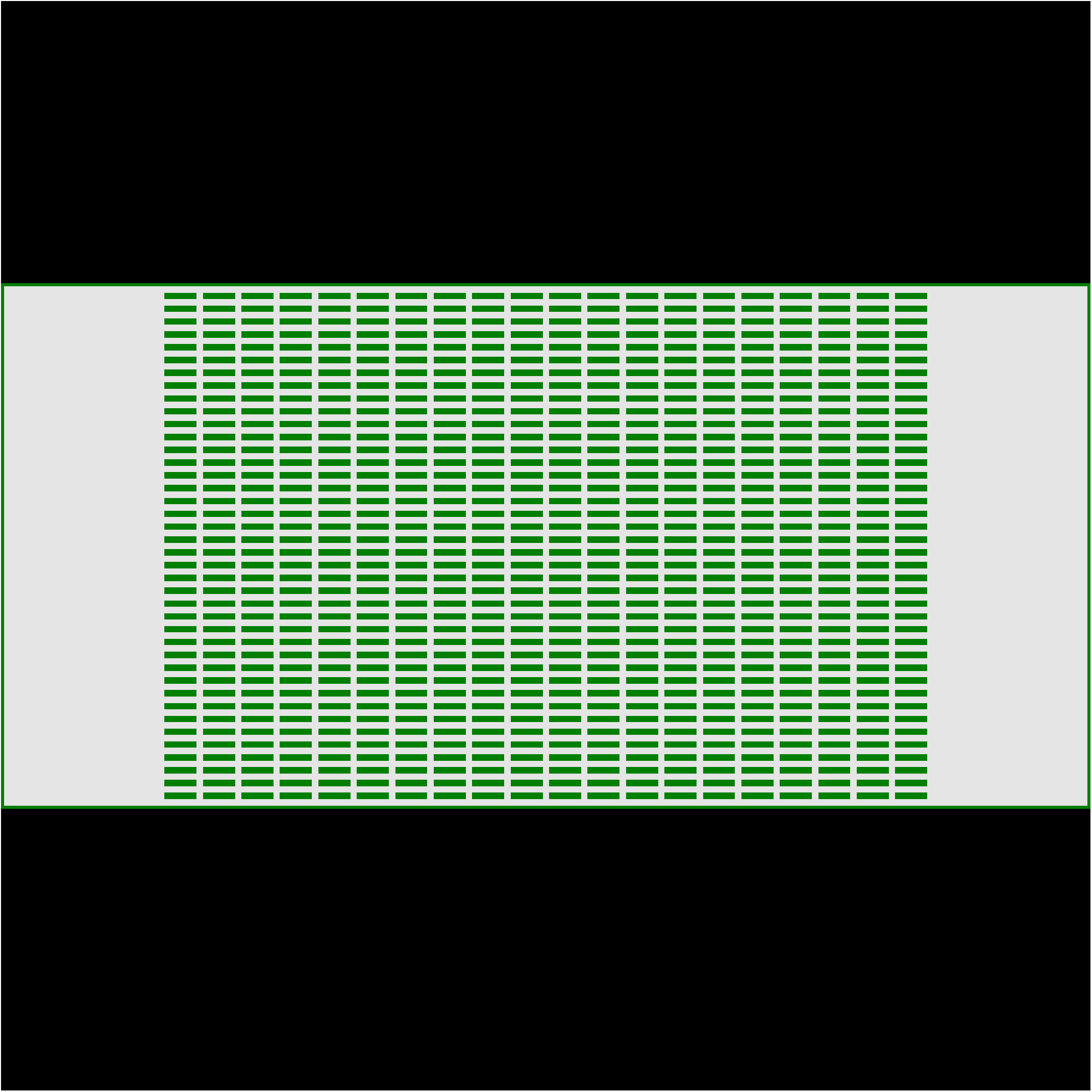}} & 128 & \phantom{0}81.2 & 27.4 & 260.7 & 197.3 & 151.2 & 151.1 & 707.6 & 711.5 & 1018.8 & 520.3 & 658.9 & 544.5 & 968.3 & 997.0 & 1018.9 & 0.9 & 0.9 \\
             &  &  & 256 & 149.3 & 57.1 & 401.5 & 412.7 & 250.4 & 249.3 & 502.5 & 484.8 & 772.2 & 319.6 & 423.1 & 337.7 & 904.0 & 991.7 & 772.3 & \textbf{1.2} & \textbf{1.3} \\
             &  &  & 512 & 335.5 & 216.2 & 473.7 & 539.2 & 440.4 & 439.6 & 375.2 & 304.7 & 684.0 & 170.3 & 276.4 & 199.3 & 848.9 & 971.0 & 684.1 & \textbf{1.2} & \textbf{1.4} \\
             
            \cmidrule(lr){2-21}
            
             & \multirow{3}{*}{\shortstack{W8:\\brc202d\\ $481 \times 530\ (43,151)$}} & \multirow{3}{*}{\includegraphics[scale=0.025]{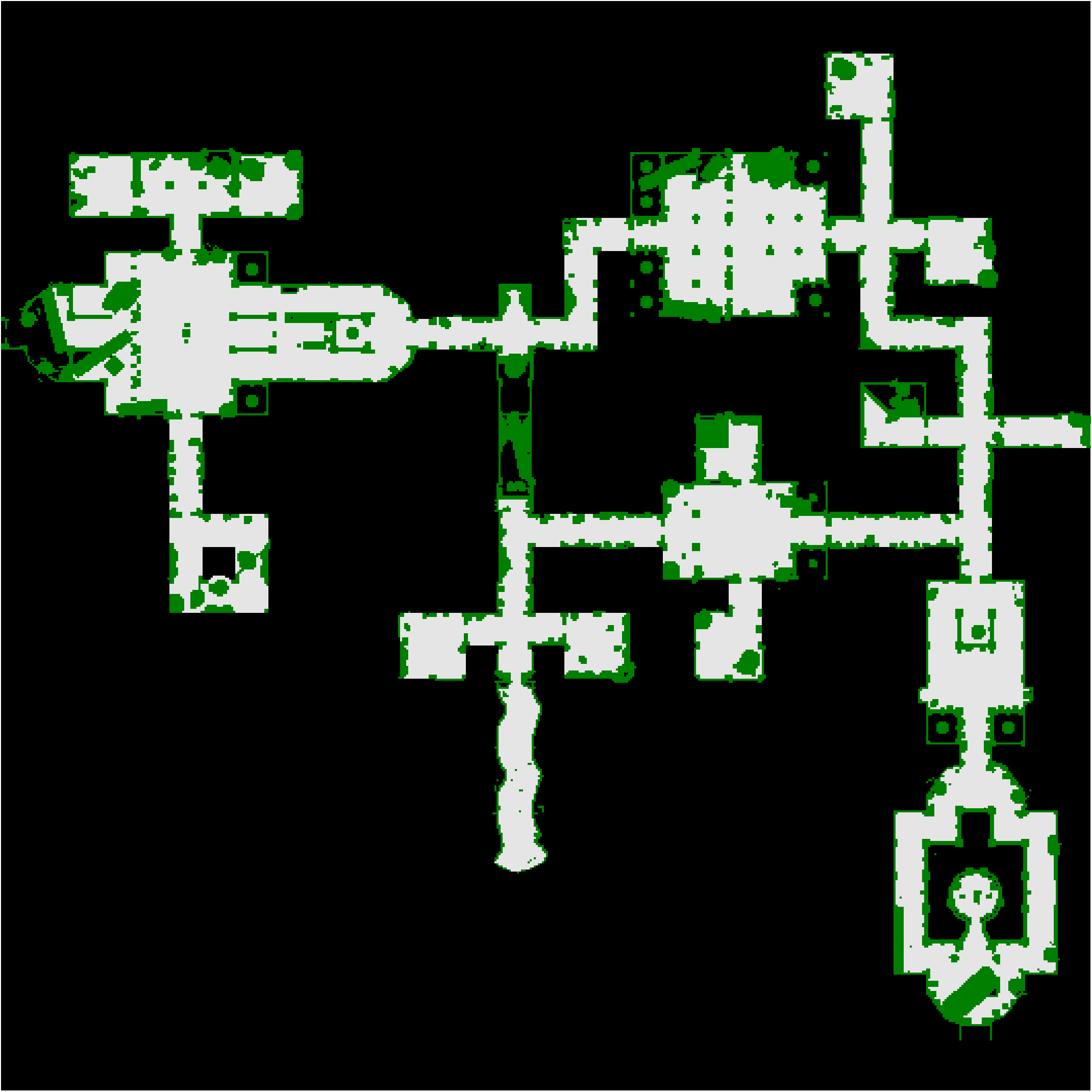}} & 128 & \phantom{0}65.1 & 25.9 & 467.7 & 391.3 & 288.4 & 288.1 & 1444.9 & 1048.3 & 1764.2 & 857.4 & 999.7 & 917.7 & 1912.6 & 1535.9 & 1764.3 & \textbf{1.1} & 0.9 \\
             &  &  & 256 & 142.0 & 56.1 & 790.1 & 831.9 & 398.1 & 396.1 & 972.5 & 573.6 & 1413.2 & 461.5 & 538.4 & 565.7 & 1762.6 & 1526.6 & 1413.3 & \textbf{1.2} & \textbf{1.1} \\
             &  &  & 512 & 309.6 & 215.7 & 818.2 & 980.6 & 546.3 & 544.8 & 780.9 & 326.4 & 1030.2 & 266.2 & 303.1 & 297.3 & 1599.1 & 1445.2 & 1030.3 & \textbf{1.6} & \textbf{1.4} \\
            \bottomrule
        \end{tabular}
   }
\end{table*}

%% file: tables/experimental_results_long.tex
\begin{table*}[t]
    \caption{Experimental results}
    \label{tab:experimental_results}
    \centering
    \resizebox{0.99\textwidth}{!}      
    {
        \begin{tabular}{@{}lcccrrrrrrrrrrrcc@{}}
            \toprule

             & & & & \multicolumn{2}{c}{$R^*$} & \multicolumn{3}{c}{$T_c\ (\si{\second})$} & \multicolumn{3}{c}{$T_p\ (\si{\second})$} & \multicolumn{3}{c}{$T_m\ (\si{\second})$} & \multicolumn{2}{c}{\textbf{Mission SU}} \\
             
            \cmidrule(lr){5-6}
            \cmidrule(lr){7-9}
            \cmidrule(lr){10-12}
            \cmidrule(lr){13-15}
            \cmidrule(lr){16-17}
            
            $M$ & \multicolumn{2}{c}{\textit{Workspace}} & $R$ & \FnOnDem & \FnCon & \FnOnDem & \FnAPF & \FnCon & \FnOnDem & \FnAPF & \FnCon & \FnOnDem & \FnAPF & \FnCon & \FnOnDem & \FnAPF \\
            
            \midrule
            
            \multirow{12}{*}{\begin{turn}{90}Quadcopter\end{turn}} & \multirow{3}{*}{\shortstack{W1:\\w\_woundedcoast\\  $578 \times 642\ (34,020)$}} & \multirow{3}{*}{\includegraphics[scale=0.025]{figures/workspaces/w_woundedcoast.pdf}} & 128 & \phantom{0}80.4 $\pm$ 3.8 & 33.8 $\pm$ \phantom{0}3.0 & 410.1 $\pm$ \phantom{0}58.4 & 193.2 $\pm$ 41.8 & 265.3 $\pm$ 40.6 & 1042.0 $\pm$ 149.5 & 895.6 $\pm$ 44.8 & 1358.0 $\pm$ 135.2 & 1452.1 $\pm$ 167.1 & 1146.5 $\pm$ \phantom{0}87.5 & 1358.2 $\pm$ 135.2 & \textbf{1.1} & 0.8 \\
             &  &  & 256 & 170.7 $\pm$ 5.7 & 82.6 $\pm$ \phantom{0}5.3 & 579.2 $\pm$ 150.2 & 417.7 $\pm$ 35.2 & 344.2 $\pm$ 53.7 & 804.2 $\pm$ 250.9 & 588.8 $\pm$ 68.1 & 962.0 $\pm$ 197.7 & 1383.4 $\pm$ 383.0 & 1083.6 $\pm$ 105.7 & 962.1 $\pm$ 197.7 & \textbf{1.4} & \textbf{1.1} \\
             &  &  & 512 & 361.5 $\pm$ 9.7 & 233.9 $\pm$ 13.2 & 639.5 $\pm$ 168.6 & 558.4 $\pm$ 20.6 & 512.8 $\pm$ 61.2 & 557.2 $\pm$ 138.9 & 395.6 $\pm$ 80.7 & 891.0 $\pm$ 183.1 & 1196.7 $\pm$ 255.4 & 1048.2 $\pm$ 108.1 & 891.0 $\pm$ 183.1 & \textbf{1.3} & \textbf{1.2} \\
             
            \cmidrule(lr){2-17}
            
             & \multirow{3}{*}{\shortstack{W2:\\Paris\_1\_256\\ $256 \times 256\ (47,240)$}} & \multirow{3}{*}{\includegraphics[scale=0.025]{figures/workspaces/Paris_1_256.pdf}} & 128 & \phantom{0}81.3 $\pm$ 2.6 & 30.3 $\pm$ \phantom{0}1.8 & 332.5 $\pm$ \phantom{0}37.8 & 158.3 $\pm$ 58.7 & 222.3 $\pm$ \phantom{0}23.7 & 962.6 $\pm$ 100.8 & 853.0 $\pm$ 28.3 & 1371.6 $\pm$ 149.4 & 1295.1 $\pm$ 118.8 & 1082.0 $\pm$ 79.6 & 1371.7 $\pm$ 149.4 & 0.9 & 0.8 \\
             &  &  & 256 & 157.2 $\pm$ 3.7 & 62.6 $\pm$ \phantom{0}1.7 & 618.5 $\pm$ \phantom{0}79.6 & 338.7 $\pm$ 20.7 & 363.3 $\pm$ \phantom{0}15.3 & 656.2 $\pm$ \phantom{0}49.4 & 573.4 $\pm$ 49.4 & 965.8 $\pm$ \phantom{0}24.9 & 1274.7 $\pm$ 110.8 & 1014.9 $\pm$ 63.8 & 965.9 $\pm$ \phantom{0}24.9 & \textbf{1.3} & \textbf{1.1} \\
             &  &  & 512 & 335.6 $\pm$ 7.0 & 194.4 $\pm$ 13.4 & 992.2 $\pm$ 113.6 & 704.6 $\pm$ 16.4 & 658.0 $\pm$ 139.1 & 455.7 $\pm$ \phantom{0}50.1 & 373.3 $\pm$ 73.2 & 929.0 $\pm$ 146.5 & 1447.9 $\pm$ 117.1 & 1195.6 $\pm$ 90.5 & 929.1 $\pm$ 146.6 & \textbf{1.6} & \textbf{1.3} \\
             
            \cmidrule(lr){2-17}
            
             & \multirow{3}{*}{\shortstack{W3:\\Berlin\_1\_256\\ $256 \times 256\ (47,540)$}} & \multirow{3}{*}{\includegraphics[scale=0.025]{figures/workspaces/Berlin_1_256.pdf}} & 128 & \phantom{0}83.2 $\pm$ \phantom{0}2.5 & 31.9 $\pm$ \phantom{0}1.7 & 296.2 $\pm$ \phantom{0}17.2 & 170.6 $\pm$ \phantom{0}90.5 & 225.4 $\pm$ \phantom{0}12.1 & 898.2 $\pm$ \phantom{0}63.9 & 739.0 $\pm$ \phantom{0}31.2 & 1334.0 $\pm$ \phantom{0}49.1 & 1194.4 $\pm$ \phantom{0}65.0 & 978.1 $\pm$ \phantom{0}50.4 & 1334.1 $\pm$ \phantom{0}49.1 & 0.9 & 0.7 \\
             &  &  & 256 & 167.4 $\pm$ 12.1 & 80.8 $\pm$ 23.1 & 598.9 $\pm$ \phantom{0}69.8 & 393.5 $\pm$ 136.7 & 417.9 $\pm$ \phantom{0}96.2 & 669.6 $\pm$ \phantom{0}48.9 & 533.2 $\pm$ \phantom{0}90.9 & 1186.8 $\pm$ 275.0 & 1268.5 $\pm$ 103.3 & 1018.1 $\pm$ 202.5 & 1186.9 $\pm$ 275.0 & \textbf{1.1} & 0.9 \\
             &  &  & 512 & 365.2 $\pm$ 34.5 & 259.8 $\pm$ 87.4 & 947.8 $\pm$ 115.1 & 911.6 $\pm$ 202.5 & 713.8 $\pm$ 280.3 & 524.9 $\pm$ 140.0 & 353.8 $\pm$ 185.2 & 1075.6 $\pm$ 381.5 & 1472.7 $\pm$ 238.9 & 1371.9 $\pm$ 372.6 & 1075.7 $\pm$ 381.5 & \textbf{1.4} & \textbf{1.3} \\
             
            \cmidrule(lr){2-17}
            
             & \multirow{3}{*}{\shortstack{W4:\\Boston\_0\_256\\ $256 \times 256\ (47,768)$}} & \multirow{3}{*}{\includegraphics[scale=0.025]{figures/workspaces/Boston_0_256.pdf}} & 128 & \phantom{0}79.6 $\pm$ 4.5 & 30.8 $\pm$ \phantom{0}0.8 & 313.3 $\pm$ \phantom{0}30.1 & 259.0 $\pm$ 30.7 & 217.8 $\pm$ 12.4 & 928.0 $\pm$ 100.6 & 730.0 $\pm$ 17.6 & 1334.8 $\pm$ 102.1 & 1241.3 $\pm$ 116.3 & 1057.4 $\pm$ \phantom{0}43.1 & 1334.9 $\pm$ 102.1 & 0.9 & 0.8 \\
             &  &  & 256 & 157.8 $\pm$ 9.3 & 66.9 $\pm$ \phantom{0}4.3 & 531.7 $\pm$ \phantom{0}61.2 & 494.6 $\pm$ 52.7 & 339.8 $\pm$ 37.9 & 625.9 $\pm$ \phantom{0}46.9 & 428.6 $\pm$ 58.8 & 949.8 $\pm$ \phantom{0}29.7 & 1157.6 $\pm$ \phantom{0}92.4 & 1017.4 $\pm$ 110.3 & 949.9 $\pm$ \phantom{0}29.7 & \textbf{1.2} & \textbf{1.1} \\
             &  &  & 512 & 345.3 $\pm$ 7.7 & 187.0 $\pm$ 10.2 & 718.1 $\pm$ 119.3 & 766.5 $\pm$ 20.1 & 501.8 $\pm$ 34.1 & 428.7 $\pm$ \phantom{0}69.2 & 225.2 $\pm$ 66.4 & 847.6 $\pm$ \phantom{0}58.4 & 1146.8 $\pm$ 182.2 & 1097.9 $\pm$ \phantom{0}92.4 & 847.7 $\pm$ \phantom{0}58.4 & \textbf{1.4} & \textbf{1.3} \\
             
            \cmidrule(lr){1-17}
            
            \multirow{12}{*}{\begin{turn}{90}TurtleBot\end{turn}} & \multirow{3}{*}{\shortstack{W5:\\maze-128-128-2\\ $128 \times 128\ (10,858)$}} & \multirow{3}{*}{\includegraphics[scale=0.025]{figures/workspaces/maze-128-128-2.pdf}} & 128 & \phantom{0}75.5 $\pm$ 7.4 & 47.8 $\pm$ \phantom{0}8.6 & 76.8 $\pm$ 19.8 & \textcolor{red}{TO} & 50.1 $\pm$ 12.2 & 791.5 $\pm$ 202.1 & \textcolor{red}{TO} & 874.6 $\pm$ 89.7 & 868.3 $\pm$ 217.7 & \textcolor{red}{TO} & 874.7 $\pm$ 89.7 & 0.9 & \textcolor{red}{TO} \\
             &  &  & 256 & 174.7 $\pm$ 7.5 & 111.6 $\pm$ 14.7 & 97.5 $\pm$ 16.9 & \textcolor{red}{TO} & 85.4 $\pm$ 24.5 & 417.6 $\pm$ \phantom{0}97.4 & \textcolor{red}{TO} & 489.8 $\pm$ 71.4 & 515.1 $\pm$ 113.7 & \textcolor{red}{TO} & 489.9 $\pm$ 71.4 & \textbf{1.1} & \textcolor{red}{TO} \\
             &  &  & 512 & 378.4 $\pm$ 6.8 & 301.8 $\pm$ 35.8 & 155.5 $\pm$ 23.9 & \textcolor{red}{TO} & 156.1 $\pm$ 40.4 & 224.3 $\pm$ \phantom{0}32.1 & \textcolor{red}{TO} & 293.4 $\pm$ 55.7 & 379.8 $\pm$ \phantom{0}54.1 & \textcolor{red}{TO} & 293.5 $\pm$ 55.7 & \textbf{1.3} & \textcolor{red}{TO} \\
             
            \cmidrule(lr){2-17}
            
             & \multirow{3}{*}{\shortstack{W6:\\den520d\\ $257 \times 256\ (28,178)$}} & \multirow{3}{*}{\includegraphics[scale=0.025]{figures/workspaces/den520d.pdf}} & 128 & \phantom{0}71.9 $\pm$ 3.2 & 31.4 $\pm$ \phantom{0}2.6 & 191.5 $\pm$ 17.7 & 138.2 $\pm$ 110.6 & 133.6 $\pm$ 24.8 & 773.1 $\pm$ 106.9 & 731.3 $\pm$ 17.4 & 991.0 $\pm$ 96.4 & 964.6 $\pm$ 117.7 & 953.9 $\pm$ 131.2 & 991.1 $\pm$ 96.4 & 0.9 & 0.9 \\
             &  &  & 256 & 153.8 $\pm$ 3.7 & 68.0 $\pm$ \phantom{0}2.2 & 307.6 $\pm$ 63.1 & 316.9 $\pm$ \phantom{0}73.8 & 217.0 $\pm$ 36.2 & 506.4 $\pm$ \phantom{0}57.4 & 463.9 $\pm$ 32.2 & 715.8 $\pm$ 37.9 & 814.0 $\pm$ 103.4 & 873.8 $\pm$ 104.2 & 715.9 $\pm$ 37.9 & \textbf{1.1} & \textbf{1.2} \\
             &  &  & 512 & 331.1 $\pm$ 6.6 & 200.4 $\pm$ 18.9 & 384.9 $\pm$ 29.2 & 371.0 $\pm$ \phantom{0}27.0 & 292.7 $\pm$ 39.5 & 398.6 $\pm$ \phantom{0}37.9 & 343.0 $\pm$ 71.4 & 519.4 $\pm$ 40.9 & 783.5 $\pm$ \phantom{0}51.7 & 814.9 $\pm$ \phantom{0}85.2 & 519.5 $\pm$ 40.9 & \textbf{1.5} & \textbf{1.6} \\
             
            \cmidrule(lr){2-17}
            
             & \multirow{3}{*}{\shortstack{W7:\\warehouse-20-40-10-2-2\\ $164 \times 340\ (38,756)$}} & \multirow{3}{*}{\includegraphics[scale=0.025]{figures/workspaces/warehouse-20-40-10-2-2.pdf}} & 128 & \phantom{0}81.2 $\pm$ 2.6 & 27.4 $\pm$ \phantom{0}1.5 & 260.7 $\pm$ 22.2 & 197.3 $\pm$ 48.4 & 151.2 $\pm$ 12.4 & 707.6 $\pm$ \phantom{0}82.5 & 711.5 $\pm$ 14.9 & 1018.8 $\pm$ \phantom{0}59.8 & 968.3 $\pm$ \phantom{0}98.3 & 997.0 $\pm$ 52.1 & 1018.9 $\pm$ \phantom{0}59.8 & 0.9 & 0.9 \\
             &  &  & 256 & 149.3 $\pm$ 8.1 & 57.1 $\pm$ \phantom{0}6.6 & 401.5 $\pm$ 31.2 & 412.7 $\pm$ 17.4 & 250.4 $\pm$ 55.6 & 502.5 $\pm$ \phantom{0}78.2 & 484.8 $\pm$ 50.4 & 772.2 $\pm$ \phantom{0}63.4 & 904.0 $\pm$ 100.4 & 991.7 $\pm$ 56.3 & 772.3 $\pm$ \phantom{0}63.4 & \textbf{1.2} & \textbf{1.3} \\
             &  &  & 512 & 335.5 $\pm$ 7.2 & 216.2 $\pm$ 16.1 & 473.7 $\pm$ 36.6 & 539.2 $\pm$ 18.8 & 440.4 $\pm$ 82.3 & 375.2 $\pm$ \phantom{0}29.5 & 304.7 $\pm$ 62.1 & 684.0 $\pm$ 111.9 & 848.9 $\pm$ \phantom{0}54.7 & 971.0 $\pm$ 55.5 & 684.1 $\pm$ 111.9 & \textbf{1.2} & \textbf{1.4} \\
             
            \cmidrule(lr){2-17}
            
             & \multirow{3}{*}{\shortstack{W8:\\brc202d\\ $481 \times 530\ (43,151)$}} & \multirow{3}{*}{\includegraphics[scale=0.025]{figures/workspaces/brc202d.pdf}} & 128 & \phantom{0}65.1 $\pm$ 3.5 & 25.9 $\pm$ \phantom{0}2.3 & 467.7 $\pm$ \phantom{0}70.7 & 391.3 $\pm$ 173.8 & 288.4 $\pm$ 53.2 & 1444.9 $\pm$ 232.0 & 1048.3 $\pm$ 81.0 & 1764.2 $\pm$ 241.4 & 1912.6 $\pm$ 273.6 & 1535.9 $\pm$ 264.6 & 1764.3 $\pm$ 241.4 & \textbf{1.1} & 0.9 \\
             &  &  & 256 & 142.0 $\pm$ 8.7 & 56.1 $\pm$ \phantom{0}6.8 & 790.1 $\pm$ 152.3 & 831.9 $\pm$ \phantom{0}50.1 & 398.1 $\pm$ 80.3 & 972.5 $\pm$ 147.1 & 573.6 $\pm$ 49.9 & 1413.2 $\pm$ \phantom{0}91.9 & 1762.6 $\pm$ 230.3 & 1526.6 $\pm$ \phantom{0}91.4  & 1413.3 $\pm$ \phantom{0}91.9 & \textbf{1.2} & \textbf{1.1} \\
             &  &  & 512 & 309.6 $\pm$ 9.5 & 215.7 $\pm$ 16.6 & 818.2 $\pm$ 112.5 & 980.6 $\pm$ \phantom{0}35.6 & 546.3 $\pm$ 43.0 & 780.9 $\pm$ \phantom{0}98.5 & 326.4 $\pm$ 93.2 & 1030.2 $\pm$ 119.8 & 1599.1 $\pm$ 163.1 & 1445.2 $\pm$ 137.0 & 1030.3 $\pm$ 119.8 & \textbf{1.6} & \textbf{1.4} \\
            \bottomrule
        \end{tabular}
   }
\end{table*}

%% file: 7_conclusion.tex
\section{Conclusion}
\label{sec:conclusion}

We have proposed a non-horizon-based centralized online multi-robot CP, where path planning and path execution happen concurrently. 
This overlapping saves significant time in completing coverage of large workspaces with hundreds of robots, thereby achieving a speedup of up to $1.6\times$ compared to its horizon-based counterparts. 
We have also validated the CP by performing simulations and real experiments. 
Though we have considered only 2D workspaces for coverage, we can seamlessly extend our CP for covering 3D workspaces with robots capable of executing motions in 3D, as the core algorithm \cite{DBLP:conf/iros/MitraS22} has already demonstrated in ROS+Gazebo simulations.